\documentclass{article}

\makeatletter

\def\Headings#1#2{\def\ps@mypagestyle{\let\@mkboth\@gobbletwo%
\def\@oddhead{\hfill {\small\sc #1} \hfill}%
\def\@oddfoot{\hfill \small\rm \thepage \hfill}%
\def\@evenhead{\hfill {\small\sc #2} \hfill}%
\def\@evenfoot{\hfill \small\rm \thepage \hfill}}%
\pagestyle{mypagestyle}}

\renewcommand\footnoterule{\kern-3\p@ \hrule \@width \textwidth \kern 2\p@}

\renewcommand{\p@enumii}{\theenumi.}
\renewcommand{\p@enumiii}{\theenumi.\theenumii.}

\def\@startsiction#1#2#3#4#5#6{\if@noskipsec \leavevmode \fi
	\par \@tempskipa #4\relax
	\@afterindenttrue
	\ifdim \@tempskipa <\z@ \@tempskipa -\@tempskipa \@afterindentfalse\fi
	\if@nobreak \everypar{}\else
	\addpenalty{\@secpenalty}\addvspace{\@tempskipa}\fi \@ifstar
	{\@ssect{#3}{#4}{#5}{#6}}{\@dblarg{\@sict{#1}{#2}{#3}{#4}{#5}{#6}}}}

\def\@sict#1#2#3#4#5#6[#7]#8{\ifnum #2>\c@secnumdepth
	\def\@svsec{}\else
	\refstepcounter{#1}\edef\@svsec{\csname the#1\endcsname}\fi
	\@tempskipa #5\relax
	\ifdim \@tempskipa>\z@
	\begingroup #6\relax
	\@hangfrom{\hskip #3\relax\@svsec.\hskip 0.1em}
	{\interlinepenalty \@M #8\par}
	\endgroup
	\csname #1mark\endcsname{#7}\addcontentsline
	{toc}{#1}{\ifnum #2>\c@secnumdepth \else
	\protect\numberline{\csname the#1\endcsname}\fi
	#7}\else
	\def\@svsechd{#6\hskip #3\@svsec #8\csname #1mark\endcsname
	{#7}\addcontentsline
	{toc}{#1}{\ifnum #2>\c@secnumdepth \else
	\protect\numberline{\csname the#1\endcsname}\fi
	#7}}\fi
	\@xsect{#5}}

\def\section{\@startsiction{section}{1}{\z@}{-7.6mm}{2.5mm}{\large\bf\raggedright}}
\def\subsection{\@startsection{subsection}{2}{\z@}{-5mm}{2mm}{\normalsize\bf\raggedright}}
\def\subsubsection{\@startsection{subsubsection}{3}{\z@}{-4.6mm}{2mm}{\bf\raggedright}}

\def\paragraph{\@startsiction{paragraph}{4}{\z@}{1.5ex plus 0.5ex minus .2ex}{-1em}{\normalsize\bf}}
\def\subparagraph{\@startsiction{subparagraph}{5}{\z@}{1.5ex plus 0.5ex minus .2ex}{-1em}{\normalsize\bf}}

\makeatother

\usepackage{amsmath,amsthm,amssymb}
\usepackage{enumerate,bbm}
\usepackage{geometry}
\usepackage{cite}
\usepackage[colorlinks,linktoc=all,linkcolor=blue]{hyperref}
\usepackage[capitalise,sort]{cleveref}
\usepackage{graphicx}
\usepackage[hypcap=false]{caption}

\Crefname{setting}{Setting}{Settings}

\renewenvironment{abstract}
{\centerline{\bf Abstract}\vspace{0.7ex}%
	\bgroup\leftskip 40pt\rightskip 40pt\small\noindent\ignorespaces}%
{\par\egroup\vskip 0.7ex}

\renewenvironment{equation*}{\begin{equation}}{\end{equation}}
\numberwithin{equation}{section}

\newtheorem{theorem}{Theorem}[section]
\newtheorem{corollary}[theorem]{Corollary}
\newtheorem{lemma}[theorem]{Lemma}
\newtheorem{proposition}[theorem]{Proposition}
\newtheorem{definition}[theorem]{Definition}
\newtheorem{remark}[theorem]{Remark}
\newtheorem{setting}[theorem]{Setting}

\newcommand{\N}{\mathbb{N}}
\newcommand{\R}{\mathbb{R}}
\renewcommand{\epsilon}{\varepsilon}

\newcommand{\ssum}[2]{\mathop{\textstyle{\sum}}_{#1}^{#2}}
\newcommand{\ID}[1]{\mathrm{id}_{#1}}
\newcommand{\IndFct}[1]{\mathbbm{1}_{#1}}
\newcommand{\ReLU}[1]{\max\{#1,0\}}

\newcommand{\NNfct}{f_{\phi}}
\newcommand{\NNfctPath}[1]{f_{\phi_{#1}}}

\newcommand{\NNfctAlt}{f_{\psi}}
\newcommand{\NNfctSpecial}{f_{\phi+\psi}}
\newcommand{\NNfctLeaky}{f_{\phi}^{\gamma}}
\newcommand{\NNfctTransf}{f_{P(\phi)}}

\newcommand{\NNfctQuad}{f_{\phi}^{\mathrm{quad}}}

\newcommand{\Loss}{\mathcal{L}}
\newcommand{\LossFull}{\Loss_{N,T,\mathcal{A}}}
\newcommand{\LossLeaky}{\Loss^{\gamma}}
\newcommand{\LossLeakyFull}{\LossFull^{\gamma}}
\newcommand{\LossQuad}{\LossFull^{\mathrm{quad}}}

\newcommand{\Grad}{\mathcal{G}}
\newcommand{\GradFull}{\Grad_{N,T,\mathcal{A}}}

\title{\Large{\bf{Landscape analysis for shallow neural networks: \\ complete classification of critical points for affine target functions}} \vskip 1mm}

\author{Patrick Cheridito\footnote{Department of Mathematics, ETH Zurich, Switzerland} \qquad Arnulf Jentzen\footnote{School of Data Science and Shenzhen Research Institute of Big Data, The Chinese University of Hong Kong, Shenzhen, China} \footnote{Applied Mathematics: Institute for Analysis and Numerics, Faculty of Mathematics and Computer Science, University of M{\"u}nster, Germany} \qquad Florian Rossmannek$^*$\footnote{corresponding author; e-mail: florian.rossmannek@math.ethz.ch}}

\Headings{Landscape analysis for shallow neural networks}{Cheridito, Jentzen, Rossmannek}

\date{}

\begin{document}

\maketitle

\begin{abstract}
In this paper, we analyze the landscape of the true loss of neural networks with one hidden layer and ReLU, leaky ReLU, or quadratic activation.
In all three cases, we provide a complete classification of the critical points in the case where the target function is affine and one-dimensional.
In particular, we show that there exist no local maxima and clarify the structure of saddle points.
Moreover, we prove that non-global local minima can only be caused by `dead' ReLU neurons.
In particular, they do not appear in the case of leaky ReLU or quadratic activation.
Our approach is of a combinatorial nature and builds on a careful analysis of the different types of hidden neurons that can occur.
\end{abstract}


\section{Introduction}

An important aspect of neural network theory in machine learning is the dynamic behavior of gradient-based training algorithms.
Although empirical evidence suggests that training is often successful, meaning that the algorithm reaches a point that is close to a global minimum of the loss function measuring the error (see, e.g., \cite{LeCun2015}), a full theoretical understanding of gradient-based methods in network models is still lacking.
One branch of recent research has been investigating the effects of overparametrization, i.e.\ using an exceedingly large number of neurons in the network model, on the convergence behavior (we refer to \cite{ChizatOyallonBach2019,AllLiSong2019} and the references therein for more details on this), but here we focus on landscape analysis of the loss surface.
This landscape analysis provides an indirect tool for studying the dynamics of gradient-based algorithms, as these dynamics are governed by the loss surface.
One goal of landscape analysis is a better understanding of the occurrence and frequency of critical points of the loss function and obtaining information about their type, that is, whether they constitute extrema, local extrema, or saddle points.
Using the hierarchical structure of networks, some partial results have been obtained; see \cite{FukuAma2000}.
Though, the choice of activation function in the network model can have a significant impact on the landscape.
For instance, it is known that the loss surface of a linear network, that is, a network with the identity function as activation, only has global minima and saddle points but no non-global local minima (see \cite{BaldiHornik1989,Kawaguchi2016}).
However, the picture becomes less clear if a nonlinearity is introduced (see \cite{ManVanZde2020,SafranShamir2018}).

In the last decade, progress has been made in this more difficult nonlinear case.
In \cite{ChoHenMatBenLeCun2015}, the loss surface has been studied by relating it to a model from statistical physics.
This way, detailed results have been obtained about the frequency and quality of local minima.
Although the findings of \cite{ChoHenMatBenLeCun2015} are theoretically insightful, their theory is based on assumptions that are not met in practice (see \cite{ChoLeCunBen2015}).
In \cite{SoudryHoffer2017}, similar results have been obtained for networks with one hidden layer with less unrealistic assumptions.
We refer to \cite{DauPasGulChoGanBen2014} for experimental findings, on which \cite{ChoHenMatBenLeCun2015,SoudryHoffer2017} is based.

Besides work studying the effects of overparametrization on gradient-based methods directly, there have also been investigations of its impact on the loss landscape.
For instance, it has been shown in \cite{SafranShamir2016} that taking larger networks increases the likelihood to start from a good initialization with small loss or from which there exists a monotonically decreasing path to a global minimum.
However, it is still not fully understood in which situations a gradient-based training algorithm follows such a path.
If the quadratic activation function is used in a network with one hidden layer, then in the overparametrized regime only global minima and strict saddle points remain, but no non-global local minima; see \cite{DuLee2018,VenBanBru2019}.
Even for deeper architectures, all non-global local minima disappear with high probability for any activation function if the width of the last hidden layer is increased
(see \cite{SoudryCarmon2016,SolJavLee2019,LivShaSha2014}) and, under some regularity assumptions on the activation, this continues to hold if any of the hidden layers is sufficiently wide and the proceeding layers have a pyramidal structure (see \cite{NguHein2017}).
However, note that these results only apply in this level of generality if the loss is measured with respect to a finite set of data.
In particular, these global minima are potentially prone to overfitting.

In contrast to the literature mentioned above, our results concern the landscape of the true loss instead of the empirical loss.
The final goal in machine learning is to minimize not only the empirical loss, but the true loss, so it is of essence to understand its landscape.
In this paper, we consider networks with a single hidden layer with (leaky) rectified linear unit (ReLU) or quadratic activation.
As an alternative to the popular theme of overparametrization, we do not impose assumptions on the network model that are not met in practice, but instead focus on {\it special target functions}.
In \cite{CheJenRieRos2022}, this strategy has been pursued with constant target functions.
In this paper, we expand the scope from constant to affine functions.
This represents a first step towards a better understanding of the true loss landscape corresponding to general target functions.

In this framework with affine target functions, we provide a complete classification of the critical points of the true loss.
We do so by unfolding the combinatorics of the problem, governed by different types of hidden neurons appearing in a network.
We find that ReLU networks admit non-global local minima regardless of the number of hidden neurons.
At the same time, it turns out that these local minima are solely caused by `dead' ReLU neurons.
In particular, for leaky ReLU networks, which are often used to avoid the problem of dead neurons, there are only saddle points and global minima.
This suggests that using leaky ReLU instead of ReLU not only makes sense to avoid issues with training itself, but also to work with a better behaved loss surface on which training takes place to begin with.
Interestingly, also for the quadratic activation, non-global local minima do not appear, which is in line with the observations in \cite{DuLee2018,VenBanBru2019} for the discretized loss but does not require overparametrization.
In addition, for networks with quadratic activation, all saddle points have a constant realization function, whereas for (leaky) ReLU networks we show that there exist saddle points with a non-constant realization.

These complete classifications in the proposed approach to consider special target functions shed new light on important aspects of gradient-based methods in the training of networks.
Knowledge of the loss surface can be transformed into results about convergence of such methods as done in, e.g., \cite{JenRie2021b}.
In a smooth setting, a recent strand of work has shown that the domain of attraction of saddle points under gradient descent has zero Lebesgue measure as long as the Hessian at the saddle points has a strictly negative eigenvalue
(see \cite{LeePanPilSimJorRecht2019,LeeSimJorRecht2016,PanPil2017}).
This indicates that it also becomes necessary to study the spectrum of the Hessian of the loss function as previously pursued in, e.g., \cite{PennBahri2017,DuLee2018}.
Using the classification in this paper, we are able to derive results about the existence of strictly negative eigenvalues of the Hessian at most of the saddle points (understood in a suitable sense because we have to deal with differentiability issues arising from the (leaky) ReLU activation).
Furthermore, the set of non-global local minima, being caused by dead ReLU neurons, consists of a single connected component in the parameter space.
In particular, these extrema are not isolated.
The behavior of (stochastic) gradient descent at not necessarily isolated local minima has been studied in, e.g., \cite{FehrGessJen2020}.

The remainder of this article is organized as follows.
The first activation function we consider is the ReLU activation in \cref{section_MainResult}.
We begin by introducing the relevant notation and definitions, including a new description of the types of hidden neurons that can appear in a ReLU network, in \cref{section_notation,section_neuron_types}.
The first main result, the classification for ReLU networks, is \cref{ThrmMain} in \cref{section_main_statement}.
The remainder of \cref{section_MainResult} is dedicated to proving the classification.
More precisely, we discuss a few important ingredients for the proof in \cref{section_methods}.
Thereafter, \cref{section_diff_loss} is devoted to differentiability and regularity properties of the loss function in view of the non-differentiability of the ReLU activation.
The heart of the proof is contained in \cref{section_crit_const,section_crit_non_const}.
Finally, we establish in \cref{section_main_proof_ID_target} a special case of \cref{ThrmMain} and deduce it in full generality afterwards in \cref{section_main_proof}.
\cref{section_leaky} is concerned with extending the classification to leaky ReLU, stated as our second main result in \cref{ThrmMainLeaky}, which heavily relies on understanding the ReLU case.
To conclude, we also classify the critical points for networks with the quadratic activation in our third main result, \cref{ThrmMainQuad} in \cref{section_quadratic}.


\section{Classification for ReLU activation}
\label{section_MainResult}

\subsection{Notation and formal problem description}
\label{section_notation}

We consider shallow networks, by which we mean networks with a single hidden layer.
For simplicity, we focus on networks with a single input and output neuron.
The set of such networks with $N \in \N$ hidden neurons can be parametrized by $\R^{3N+1}$.
We begin by describing the problem for the ReLU activation function $x \mapsto \ReLU{x}$.
We will always write an element $\phi \in \R^{3N+1}$ as $\phi = (w,b,v,c)$, where $w,b,v \in \R^N$ and $c \in \R$.
The realization of the network $\phi$ with ReLU activation is the function $\NNfct \in C(\R,\R)$ given by
\begin{equation*}
	\NNfct(x) = c + \ssum{j=1}{N} v_j \ReLU{w_jx+b_j}.
\end{equation*}%
We suppose that the objective is to approximate an affine function on an interval $[T_0,T_1]$ in the $L^2$-norm.
In other words, given $\mathcal{A}=(\alpha,\beta) \in \R^2$ and $T=(T_0,T_1) \in \R^2$, one tries to minimize the loss function $\LossFull \in C(\R^{3N+1},\R)$ given by
\begin{equation*}
	\LossFull(\phi) = \int_{T_0}^{T_1} ( \NNfct(x) - \alpha x - \beta )^2 \, dx.
\end{equation*}%
The purpose of the first half of this paper is to classify the critical points of the loss function $\LossFull$.
Since the ReLU function is not differentiable at 0, we work with the generalized gradient $\GradFull \colon \R^{3N+1} \rightarrow \R^{3N+1}$ of the loss obtained by taking right-hand partial derivatives;
\begin{equation*}
	(\GradFull(\phi))_k = \lim_{h \downarrow 0} \frac{\LossFull(\phi+he_k) - \LossFull(\phi)}{h}
\end{equation*}%
for all $k \in \{1,\dots,3N+1\}$, where $e_k$ is the $k^{th}$ unit vector in $\R^{3N+1}$.
The function $\GradFull$ is defined on the entire parameter space $\R^{3N+1}$ and agrees with the gradient of $\LossFull$ if the latter exists.
We verify this and study regularity properties of $\LossFull$ more thoroughly in \cref{section_diff_loss}.

\begin{definition}
\label{def_crit_point}
	Let $N \in \N$ and $\mathcal{A},T \in \R^2$.
Then we call $\phi \in \R^{3N+1}$ a critical point of $\LossFull$ if $\GradFull(\phi) = 0$ and a saddle point if it is a critical point but not a local extremum.\footnote{We consider non-strict local extrema, i.e.\ $\phi$ is a local minimum (maximum) of $\LossFull$ if $\LossFull(\phi) \leq \LossFull(\psi)$ ($\geq$) for all $\psi$ in an open neighborhood of $\phi$, allowing equality $\LossFull(\phi) = \LossFull(\psi)$.}
\end{definition}

It can be shown that if $\phi$ is a critical point of $\LossFull$, then 0 belongs to the limiting sub-differential of $\LossFull$; see\footnote{In \cite{EbeJenRieWei2021}, the authors use a different generalization of the gradient, which can be obtained by taking left-hand partial derivatives. However, if $\GradFull$ is zero at some $\phi$, then its left-hand analog is also zero at $\phi$, so \cite[Prop.\ 2.12]{EbeJenRieWei2021} is applicable.} \cite[Prop.\ 2.12]{EbeJenRieWei2021}.
With \cref{def_crit_point}, it is not immediately clear whether all local extrema are critical points.
However, we will show that this is the case by demonstrating that local extrema are points of differentiability of the loss function.
In particular, \cref{def_crit_point} is well-suited for our purposes.
The next notion relates the outer bias, i.e., the coordinate $c$, to the target function $x \mapsto \alpha x + \beta$.

\begin{definition}
	Let $N \in \N$, $\phi = (w,b,v,c) \in \R^{3N+1}$, $\mathcal{A}=(\alpha,\beta) \in \R^2$, and $T=(T_0,T_1) \in \R^2$.
Then we say that $\phi$ is $(T,\mathcal{A})$-centered if $c = \frac{\alpha}{2}(T_0+T_1)+\beta$.
\end{definition}

To motivate this definition, note that $\frac{\alpha}{2}(T_0+T_1)+\beta$ is the best constant $L^2$-approximation of the function $[T_0,T_1] \rightarrow \R$, $x \mapsto \alpha x + \beta$.


\subsection{Different types of hidden neurons}
\label{section_neuron_types}

In this section, we introduce a few notions that describe how the different hidden neurons in a network are contributing to the realization function.
In the definition below, we introduce sets $I_j$, which are defined such that $[T_0,T_1] \backslash I_j$ is the interval on which the output of the $j^{th}$ hidden neuron is rendered zero by the ReLU activation.

\begin{definition}
\label{def_types_of_neurons}
	Let $N \in \N$, $\phi = (w,b,v,c) \in \R^{3N+1}$, $j \in \{1,\dots,N\}$, and $T_0,T_1 \in \R$ such that $T_0<T_1$.
Then we denote by $I_j$ the set given by $I_j = \{x \in [T_0,T_1] \colon w_jx+b_j \geq 0 \}$, we say that the $j^{th}$ hidden neuron of $\phi$ is\newline
\begin{minipage}{0.59\linewidth}
\vspace{0.2cm}
\begin{itemize}\itemsep = -0.2em
\item {\it inactive} if $I_j = \emptyset$,
\item {\it semi-inactive} if $\# I_j = 1$,
\item {\it semi-active} if $w_j = 0 < b_j$,
\item {\it active} if $w_j \ne 0 < b_j + \max_{k \in \{0,1\}} w_jT_k$,
\item {\it type-1-active} if $w_j \ne 0 \leq b_j + \min_{k \in \{0,1\}} w_jT_k$,
\item {\it type-2-active} if $\emptyset \ne I_j \cap (T_0,T_1) \ne (T_0,T_1)$,
\item {\it degenerate} if $|w_j| + |b_j| = 0$,
\item {\it non-degenerate} if $|w_j| + |b_j| > 0$,
\item {\it flat} if $v_j = 0$,
\item {\it non-flat} if $v_j \ne 0$,
\end{itemize}
and we say that $t \in \R$ is the breakpoint of the $j^{th}$ hidden \newline neuron of $\phi$ if $w_j \ne 0 = w_jt+b_j$.
\end{minipage}
\hfill
\begin{minipage}{0.4\linewidth}
	\vspace{-0.2cm}
	\centering
	\includegraphics[width=0.98\linewidth]{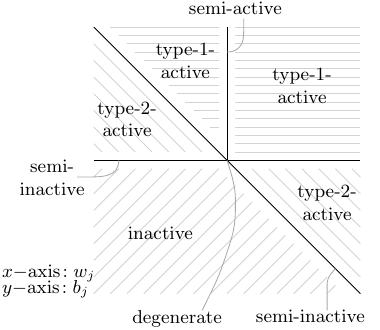}
	\vspace{-0.2cm}
	\captionof{figure}{Regions\protect\footnotemark with different types of a hidden neuron as seen in the $(w_j,b_j)$-plane.}
	\label{fig_types_of_neurons}
\end{minipage}
\end{definition}

Let us briefly motivate these notions.\footnotetext{\cref{fig_types_of_neurons} shows the case $T_0 = 0$, $T_1 = 1$. The general case is obtained by a shear transformation.}
Every hidden neuron is exactly one of: inactive, semi-inactive, semi-active, active, or degenerate.
Moreover, observe that $I_j$ is always an interval.

For an inactive neuron, applying the ReLU activation function yields the constant zero function on $[T_0,T_1]$.
The breakpoint $t_j$ might not exist (if $w_j = 0$ and $b_j < 0$), or it might exist and lie outside of $[T_0,T_1]$ with $t_j < T_0$ if $w_j < 0$ and $t_j > T_1$ if $w_j > 0$.
Note that inactivity is a stable condition in the sense that a small perturbation of an inactive neuron remains inactive.

Applying the ReLU activation to a semi-inactive neuron also yields the constant zero function on $[T_0,T_1]$.
But in this case, a breakpoint must exist and be equal to one of the endpoints $T_0,T_1$ (which one depends on the sign of $w_j$ similarly to the inactive case).
However, a perturbation of a semi-inactive neuron may yield a (semi-)inactive or a type-2-active neuron; see \cref{fig_types_of_neurons}.
In this sense, semi-inactive neurons are boundary cases.

The realization of a semi-active neuron is also constant, but not necessarily zero since the corresponding interval $I_j$ is $[T_0,T_1]$.
As can be seen from \cref{fig_types_of_neurons}, perturbing a semi-active neuron always yields a semi- or type-1-active neuron.

Non-flat active neurons provide a non-constant contribution to the overall realization function.
Note that a hidden neuron is active exactly if it is type-1- or type-2-active.
These two types distinguish whether the breakpoint $t_j$, which exists in either case, lies outside or inside the interval $(T_0,T_1)$ and, hence, whether the contribution of the neuron is affine (corresponding to $I_j = [T_0,T_1]$) or piecewise affine (corresponding to $I_j = [T_0,t_j]$ or $I_j = [t_j,T_1]$).
Type-1 and type-2-active neurons both form two connected components in the $(w_j,b_j)$-plane; see \cref{fig_types_of_neurons}.
A perturbation of an active neuron remains active.

The case $w_j = 0 = b_j$ is called degenerate because it leads to problems with differentiability.
Perturbing a degenerate neuron may yield any of the other types of neurons.

Lastly, a flat neuron also does not contribute to the overall realization, but the reason for this lies between the second and third layer and not between the first and second one, which is why this case deserves a separate notion.


\subsection{Classification of the critical points of the loss function}
\label{section_main_statement}

Now, we are ready to provide a classification of the critical points of the loss function.

\begin{theorem}
\label{ThrmMain}
	Let $N \in \N$, $\phi = (w,b,v,c) \in \R^{3N+1}$, $\mathcal{A}=(\alpha,\beta) \in \R^2$, and $T=(T_0,T_1) \in \R^2$ satisfy $\alpha \ne 0$ and $0 \leq T_0<T_1$.
Then the following hold:
\begin{enumerate}[\rm (I)]\itemsep = 0em

\item\label{ThrmMainLocMax} $\phi$ is not a local maximum of $\LossFull$.

\item\label{ThrmMainDiff} If $\phi$ is a critical point or a local extremum of $\LossFull$, then $\LossFull$ is differentiable at $\phi$ with gradient $\nabla \LossFull (\phi) = 0$.

\item\label{ThrmMainLocMin} $\phi$ is a non-global local minimum of $\LossFull$ if and only if $\phi$ is $(T,\mathcal{A})$-centered and, for all $j \in \{1,\dots,N\}$, the $j^{th}$ hidden neuron of $\phi$ is
\begin{enumerate}[\rm (a)]\itemsep = 0em

\item\label{ThrmMainLocMinInact} inactive,

\item\label{ThrmMainLocMinSemiIn1} semi-inactive with $I_j = \{T_0\}$ and $\alpha v_j>0$, or

\item\label{ThrmMainLocMinSemiIn2} semi-inactive with $I_j = \{T_1\}$ and $\alpha v_j<0$.

\end{enumerate}

\item\label{ThrmMainSaddle} $\phi$ is a saddle point of $\LossFull$ if and only if $\phi$ is $(T,\mathcal{A})$-centered, $\phi$ does not have any type-1-active neurons, $\phi$ does not have any non-flat semi-active neurons, $\phi$ does not have any non-flat degenerate neurons, and exactly one of the following two items holds:
\begin{enumerate}[\rm (a)]\itemsep = 0em

\item\label{ThrmMainSaddleTriv} $\phi$ does not have any type-2-active neurons and there exists $j \in \{1,\dots,N\}$ such that the $j^{th}$ hidden neuron of $\phi$ is
\begin{enumerate}[\rm (i)]\itemsep = 0em
\item\label{ThrmMainSaddleSemiAc} flat semi-active,

\item\label{ThrmMainSaddleSemiIn1} semi-inactive with $I_j = \{T_0\}$ and $\alpha v_j \leq 0$,

\item\label{ThrmMainSaddleSemiIn2} semi-inactive with $I_j = \{T_1\}$ and $\alpha v_j \geq 0$, or

\item\label{ThrmMainSaddleDeg} flat degenerate.

\end{enumerate}

\item\label{ThrmMainSaddleNonTriv} There exists $n \in \{2,4,6,\dots\}$ such that $(\bigcup_{j \in \{1,\dots,N\},\, w_j \ne 0} \{-\frac{b_j}{w_j}\}) \cap (T_0,T_1) = \bigcup_{i=1}^{n} \{T_0+\frac{i(T_1-T_0)}{n+1}\}$ and, for all $j \in \{1,\dots,N\}$, $i \in \{1,\dots,n\}$ with $w_j \ne 0 = b_j + w_j(T_0+\frac{i(T_1-T_0)}{n+1})$, it holds that $\mathrm{sign}(w_j) = (-1)^{i+1}$ and $\ssum{k \in \{1,\dots,N\},\, w_k \ne 0 = b_k + w_k(T_0+\frac{i(T_1-T_0)}{n+1})}{} v_k w_k = \frac{2\alpha}{n+1}$.

\end{enumerate}

\item\label{ThrmMainFctTriv} If $\phi$ is a non-global local minimum of $\LossFull$ or a saddle point of $\LossFull$ without type-2-active neurons, then $\NNfct(x) = \frac{\alpha}{2}(T_0+T_1)+\beta$ for all $x \in [T_0,T_1]$.

\item\label{ThrmMainFctNonTriv} If $\phi$ is a saddle point of $\LossFull$ with at least one type-2-active neuron, then there exists $n \in \{2,4,6,\dots\}$ such that $n \leq N$ and, for all $i \in \{0,\dots,n\}$, $x \in [T_0+\frac{i(T_1-T_0)}{n+1},T_0+\frac{(i+1)(T_1-T_0)}{n+1}]$, one has
\begin{equation*}
	\NNfct(x) = \alpha x + \beta - \frac{(-1)^i \alpha}{n+1} \Big( x - T_0 - \frac{(i+\frac{1}{2})(T_1-T_0)}{n+1} \Big).
\end{equation*}%

\end{enumerate}
\end{theorem}

\cref{ThrmMain}.\eqref{ThrmMainSaddleNonTriv} says that the set of breakpoints of all type-2-active neurons agrees with the set of $n$ equally spaced points $T_0 < q_1 < \dots < q_n < T_1$.
Furthermore, for any type-2-active neuron with breakpoint $q_i$, the sign of the coordinate $w$ is given by $(-1)^{i+1}$.
Lastly, the sum of $v_kw_k$, where $k$ ranges over all type-2-active neurons with breakpoint $q_i$, is equal to $\frac{2\alpha}{n+1}$.
The term $v_kw_k$ is the contribution of the $k^{th}$ hidden neuron to the slope of the realization.

\begin{remark}
\label{rem_proper_crit}
	Note that, by \cref{ThrmMain}.\eqref{ThrmMainDiff}, all local extrema and all critical points of $\LossFull$, which we defined as zeros of $\GradFull$, are actually critical points of $\LossFull$ in the classical sense, i.e.\ points of differentiability of $\LossFull$ with vanishing gradient.
In particular, the classification in \cref{ThrmMain} turns out to be a classification of the critical points in the classical sense as well.
\end{remark}

\begin{remark}
	Gradient Descent-type algorithms typically use generalized gradients to train ReLU networks.
For instance, they might compute $\Grad$, its left-hand analog, the average of the two, or quantities obtained by artificially defining the derivative of the ReLU function at 0.
For each of these versions, a similar classification of critical points could be derived.
\end{remark}

\cref{ThrmMain}.\eqref{ThrmMainFctTriv} shows that any non-global local minimum has the constant realization $\frac{\alpha}{2}(T_0+T_1)+\beta$.
In particular, there is only one value that the loss function can take at non-global local minima.
Similarly, it follows from \cref{ThrmMain}.\eqref{ThrmMainFctNonTriv} that a saddle point can lead to exactly one of $\lfloor N/2 \rfloor+1$ possible loss values.

\begin{corollary}
\label{CorLossValues}
	Let $N \in \N$, $\mathcal{A} = (\alpha,\beta) \in \R^2$, and $T = (T_0,T_1) \in \R^2$ satisfy $0 \leq T_0 < T_1$, and assume that $\phi \in \R^{3N+1}$ is a critical point of $\LossFull$.
Then the following hold:
\begin{enumerate}[\rm (i)]\itemsep = 0em

\item If $\phi$ is a non-global local minimum of $\LossFull$, then $\LossFull(\phi) = \frac{1}{12}\alpha^2(T_1-T_0)^3$.

\item If $\phi$ is a saddle point of $\LossFull$, then there exists $n \in \{0,2,4,\dots\}$ such that $n \leq N$ and $\LossFull(\phi) = \frac{1}{12(n+1)^4} \alpha^2 (T_1-T_0)^3$.

\end{enumerate}
\end{corollary}

Formally, \cref{CorLossValues} only follows from \cref{ThrmMain} for $\alpha \ne 0$.
But for $\alpha = 0$ it holds trivially since for constant target functions there exist no critical points other than global minima (see \cite{CheJenRieRos2022}).


\subsection{Ingredients for the proof of the classification}
\label{section_methods}

As a first step, let us provide a simple argument to establish \cref{ThrmMain}.\eqref{ThrmMainLocMax}.

\begin{lemma}
\label{LemNoMaxima}
	Let $N \in \N$, $\mathcal{A} \in \R^2$, and $T=(T_0,T_1) \in \R^2$ satisfy $T_0<T_1$.
Then $\LossFull$ does not have any local maxima.
\end{lemma}

\begin{proof}
	Write $\mathcal{A} = (\alpha,\beta)$.
The lemma directly follows from the simple fact that
\begin{equation*}
	\LossFull(w,b,v,c) = \int_{T_0}^{T_1} \Big(c + \ssum{j=1}{N} v_j \ReLU{w_jx+b_j} - \alpha x - \beta \Big)^2 dx
\end{equation*}%
is strictly convex in $c$.
\end{proof}

As a consequence of this lemma, whenever we want to show that a critical point $\phi$ is a saddle point, it suffices to show that it is not a local minimum, that is, it suffices to show that, in every neighborhood of $\phi$, $\Loss$ attains a value that is below $\Loss(\phi)$.

\begin{remark}
\label{RemNoMaxima}
	The previous proof only used linearity of the realization function in the $c$-coordinate and strict convexity of the square function.
In particular, the same argument shows that the square loss never has local maxima regardless of the target function, the activation function, and the architecture of the network.
\end{remark}

Let us now provide a sketch of the proofs to come.
Instead of proving \cref{ThrmMain} directly, we first assume that the affine target function is the identity on the interval $[0,1]$, corresponding to the special case $T_0 = \beta = 0$ and $T_1 = \alpha = 1$ in \cref{ThrmMain}.
Afterwards, we will verify that the general case can always be reduced to this one.
For convenience of notation, we assume the following convention to hold throughout the remainder of \cref{section_MainResult}.

\begin{setting}
	Fix $N \in \N$ and denote $\Loss = \Loss_{N,(0,1),(1,0)}$ and $\Grad = \Grad_{N,(0,1),(1,0)}$.
We say that a network $\phi \in \R^{3N+1}$ is centered if it is $((0,1),(1,0))$-centered.
\end{setting}

The generalized gradient $\Grad$ was defined in terms of the right-hand partial derivatives of $\Loss$.
These are given by
\begin{equation*}
\begin{split}
	\frac{\partial^{+}}{\partial w_j} \Loss(\phi) &= 2 v_j \int_{I_j} x(\NNfct(x)-x) dx, \\
	\frac{\partial^{+}}{\partial b_j} \Loss(\phi) &= 2 v_j \int_{I_j} (\NNfct(x)-x) dx,
\end{split}
\qquad
\begin{split}
	\frac{\partial^{+}}{\partial v_j} \Loss(\phi) &= 2 \int_{I_j} (w_jx+b_j)(\NNfct(x)-x) dx, \\
	\frac{\partial^{+}}{\partial c} \Loss(\phi) &= 2 \int_{0}^{1} (\NNfct(x)-x) dx.
\end{split}
\end{equation*}%
Regularity properties of the loss function will be discussed in detail in the next section.
We will see then that these right-hand partial derivatives are proper partial derivatives if the $j^{th}$ hidden neuron is flat or non-degenerate.
If these partial derivatives are zero, then we encounter the system of equations
\begin{equation}
\label{LossGradient}
\begin{split}
	0 &= 2 v_j \int_{I_j} x(\NNfct(x)-x) dx, \\
	0 &= 2 v_j \int_{I_j} (\NNfct(x)-x) dx, \\
	0 &= 2 \int_{I_j} (w_jx+b_j)(\NNfct(x)-x) dx, \\
	0 &= 2 \int_{0}^{1} (\NNfct(x)-x) dx,
\end{split}
\end{equation}%
from which we deduce that any non-flat non-degenerate neuron of a critical point or local extremum $\phi$ satisfies
\begin{equation}
\label{CritCharacteristic}
	\int_{I_j} (\NNfct(x)-x) dx = 0 = \int_{I_j} x(\NNfct(x)-x) dx.
\end{equation}%
This simple observation will be used repeatedly in the proof of \cref{ThrmMain}.
Moreover, for a type-1-active neuron (for which $I_j = [0,1]$), \eqref{CritCharacteristic} is even satisfied if the neuron is flat as can be seen from the third and fourth line of \eqref{LossGradient}.
Here is an example of how \eqref{CritCharacteristic} can be employed:
note that any affine function $f \colon [0,1] \rightarrow \R$ satisfying
\begin{equation}
\label{ExampleArgumentMethod}
	\int_{0}^{1} (f(x)-x) dx = 0 = \int_{0}^{1} x(f(x)-x) dx
\end{equation}%
necessarily equals the identity on $[0,1]$.
Thus, if $\phi$ is a critical point or local extremum of $\Loss$ for which $\NNfct$ is affine and if $\phi$ admits a type-1-active or non-flat semi-active neuron (so that $I_j = [0,1]$), then we obtain from \eqref{CritCharacteristic} that $\phi$ is a global minimum.
If $\NNfct$ is not affine, we will be able to develop similar arguments for each affine piece of $\NNfct$.
In this case, we will obtain a system of equations from \eqref{LossGradient} that intricately describes the combinatorics of the realization function.


\subsection{Differentiability of the loss function}
\label{section_diff_loss}

Since the ReLU function is not differentiable at 0, the loss function is not everywhere differentiable.
However, a simple argument establishes that $\Loss$ is differentiable at any of its global minima as the following lemma shows.

\begin{lemma}
\label{LemDiffGlobalMinima}
	Let $\phi \in \R^{3N+1}$.
If $\NNfct(x) = x$ for all $x \in [0,1]$, then $\Loss$ is differentiable at $\phi$.
\end{lemma}

\begin{proof}
	It is well known that the realization function $\R^{3N+1} \rightarrow C([0,1],\R), ~\phi \mapsto \NNfct|_{[0,1]}$ is locally Lipschitz continuous if $C([0,1],\R)$ is equipped with the supremums norm (see, e.g., \cite{PeteRasVoigt2020}).
Thus, there is a constant $L>0$ depending only on $N$ and $\phi$ with $|\NNfctSpecial(x)-\NNfct(x)| \leq L\|\psi\|$ uniformly on $[0,1]$ for all $\psi$ sufficiently close to $\phi$.
Then
\begin{equation*}
	\frac{\Loss(\phi+\psi)-\Loss(\phi)}{\|\psi\|} = \frac{1}{\|\psi\|} \int_0^1 (\NNfctSpecial(x)-\NNfct(x))^2 dx \leq L^2 \|\psi\|,
\end{equation*}%
which shows that $\Loss$ is differentiable at $\phi$.
\end{proof}

The next result shows that there even are regions in the parameter space where $\Loss$ is infinitely often differentiable in spite of the ReLU activation.

\begin{lemma}
\label{LemLossSmooth}
	The loss function $\Loss$ is everywhere analytic in $(v,c)$.
Moreover, if the $j^{th}$ hidden neuron of $\phi \in \R^{3N+1}$ is inactive, semi-active, or type-1-active with breakpoint neither 0 nor 1 for some $j \in \{1,\dots,N\}$, then $\Loss$ is also analytic in $(w_j,b_j,v,c)$ in a neighborhood of $\phi$, and mixed partial derivatives of any order can be obtained by differentiating under the integral.
In particular,
\begin{equation*}
\begin{split}
	\frac{\partial}{\partial w_j} \Loss(\phi) &= 2 v_j \int_{I_j} x(\NNfct(x)-x) dx, \\
	\frac{\partial}{\partial b_j} \Loss(\phi) &= 2 v_j \int_{I_j} (\NNfct(x)-x) dx,
\end{split}
\qquad
\begin{split}
	\frac{\partial}{\partial v_j} \Loss(\phi) &= 2 \int_{I_j} (w_jx+b_j)(\NNfct(x)-x) dx, \\
	\frac{\partial}{\partial c} \Loss(\phi) &= 2 \int_{0}^{1} (\NNfct(x)-x) dx.
\end{split}
\end{equation*}%
\end{lemma}

\begin{proof}
	For the first part, note that $\Loss$ is a polynomial in the coordinates $(v,c)$.
Secondly, assume that the $j^{th}$ hidden neuron of $\phi^0 \in \R^{3N+1}$ is inactive.
Then for all $\phi$ in a sufficiently small neighborhood of $\phi^0$ and all $x \in [0,1]$ we have $\ReLU{w_jx+b_j} = 0$.
Hence, $\Loss$ is constant in the coordinates $(w_j,b_j)$ near $\phi^0$ and it is a polynomial in $(w_j,b_j,v,c)$.
Thirdly, assume that the $j^{th}$ hidden neuron of $\phi^0$ is semi-active or type-1-active with breakpoint neither 0 nor 1.
Then for all $\phi$ in a sufficiently small neighborhood of $\phi^0$ and all $x \in [0,1]$ we have $\ReLU{w_jx+b_j} = w_jx+b_j$.
In particular, $\Loss$ is a polynomial in the coordinates $(w_j,b_j,v,c)$ near $\phi^0$.
The statement about differentiating under the integral follows from dominated convergence.
\end{proof}

In regions of the parameter space not covered by \cref{LemLossSmooth}, we cannot guarantee as much regularity of the loss function, but we can still hope for differentiability.
Indeed, we already noted in the proof of \cref{LemDiffGlobalMinima} that the realization function $\R^{3N+1} \rightarrow C([0,1],\R), ~\phi \mapsto \NNfct|_{[0,1]}$ is locally Lipschitz continuous.
So, it follows from Rademacher's theorem that $\Grad$ is, in fact, equal to the true gradient $\nabla \Loss$ of $\Loss$ almost everywhere.
In the next result, we obtain insights about the measure-zero set on which $\Grad$ may not be the true gradient.

\begin{lemma}
\label{LemLossDiff}
	For all $j \in \{1,\dots,N\}$, the right-hand partial derivatives $\partial^+ \Loss(\phi) / \partial w_j$ and $\partial^+ \Loss(\phi) / \partial b_j$ exist everywhere and are given by
\begin{equation*}
	\frac{\partial^+}{\partial w_j} \Loss(\phi) = 2 v_j \int_{I_j} x(\NNfct(x)-x) dx \quad \text{and} \quad \frac{\partial^+}{\partial b_j} \Loss(\phi) = 2 v_j \int_{I_j} (\NNfct(x)-x) dx.
\end{equation*}%
Moreover, if the $j^{th}$ hidden neuron is flat or non-degenerate, then $\Loss$ is differentiable in $(w_j,b_j,v,c)$ and, in particular, the right-hand partial derivatives $\partial^+ \Loss(\phi) / \partial w_j$ and $\partial^+ \Loss(\phi) / \partial b_j$ are proper partial derivatives.
\end{lemma}

\begin{proof}
	Let $\phi \in \R^{3N+1}$ be arbitrary and denote by $\phi_h$, $h = (h^1,h^2) \in \R^2$, the network with the same coordinates as $\phi$ except in the $j^{th}$ hidden neuron, where $\phi_h$ has coordinates $w_j + h^1$ and $b_j + h^2$.
We use the notation $I_j^h$ for the interval $I_j$ associated to $\phi_h$ and denote
\begin{equation*}
	\epsilon = \Loss(\phi_h) - \Loss(\phi) - 2 v_j h^1 \int_{I_j} x(\NNfct(x)-x) dx - 2 v_j h^2 \int_{I_j} (\NNfct(x)-x) dx.
\end{equation*}%
The proof is complete if we can show that $\epsilon$ goes to zero faster than $(h^1,h^2)$.
To do that, we estimate the two terms of the last line of
\begin{equation*}
\begin{split}
	\epsilon &= \int_{0}^{1} (\NNfctPath{h}(x)-\NNfct(x))^2 dx + 2 \int_{0}^{1} (\NNfctPath{h}(x)-\NNfct(x))(\NNfct(x)-x) dx - 2 v_j \int_{I_j} (h^1x+h^2)(\NNfct(x)-x) dx \\
	&= \int_{0}^{1} (\NNfctPath{h}(x)-\NNfct(x))^2 dx + 2 v_j \int_{0}^{1} (w_jx+b_j+h^1x+h^2) (\NNfct(x)-x) (\IndFct{I_j^h}(x) - \IndFct{I_j}(x)) dx.
\end{split}
\end{equation*}%
To control the first term, we use local Lipschitz continuity of the realization function, which yields a constant $L>0$ depending only on $\phi$ so that $|\NNfctPath{h}(x)-\NNfct(x)| \leq L(|h^1|+|h^2|)$ uniformly on $[0,1]$ for all sufficiently small $h$.
To estimate the second term, we note that the absolute value of $\IndFct{I_j^h} - \IndFct{I_j}$ is the indicator function of the symmetric difference $I_j \triangle I_j^h$.
By definition of these sets, we obtain the bound $|w_jx+b_j| \leq |h^1x + h^2|$ for any $x \in I_j \triangle I_j^h$.
This yields
\begin{equation*}
	\frac{|\epsilon|}{|h^1|+|h^2|} \leq L^2 (|h^1|+|h^2|) + 4 |v_j| \int_{0}^{1} |\NNfct(x)-x| \IndFct{I_j \triangle I_j^h}(x) dx.
\end{equation*}%
The term $L^2 (|h^1|+|h^2|)$ vanishes as $h \rightarrow 0$.
We need to argue that the second term also vanishes as $h \rightarrow 0$.
If the $j^{th}$ hidden neuron is flat, then the second term is trivially zero.
On the other hand, if the $j^{th}$ hidden neuron is non-degenerate, then the Lebesgue measure of $I_j \triangle I_j^h$ tends to zero as $h \rightarrow 0$.
Thus, in this case, the integral also vanishes as $h \rightarrow 0$.
If the $j^{th}$ hidden neuron is non-flat degenerate, then we consider the directional derivatives from the right, i.e.\ with $h^1,h^2 \downarrow 0$.
But then $I_j = [0,1] = I_j^h$, so $\IndFct{I_j \triangle I_j^h}$ is constantly zero.
\end{proof}

It is well known that a multivariate function is continuously differentiable if it has continuous partial derivatives.
The following result is a slight extension for the loss function $\Loss$.

\begin{lemma}
\label{LemLossContDiff}
	The loss function $\Loss$ is continuously differentiable on the set of networks without degenerate neurons.
In addition, $\Loss$ is differentiable at networks without non-flat degenerate neurons.
\end{lemma}

\begin{proof}
	The preceding two results established existence of all partial derivatives of first order at networks without degenerate neurons.
Furthermore, these partial derivatives are continuous in the network parameters.
This is clear for $(v,c)$ and it also holds for $(w,b)$ because the endpoints of $I_j$ vary continuously in $w_j$ and $b_j$ as long as not both are zero.
This concludes the first statement.

To prove that $\Loss$ is still differentiable if flat degenerate neurons appear, assume without loss of generality that the first $M \leq N$ hidden neurons of $\phi \in \R^{3N+1}$ are flat degenerate and the remaining $N-M$ hidden neurons are non-degenerate.
Denote by $\phi_1 \in \R^{3M+1}$ the network comprised of the first $M$ hidden neurons of $\phi$ (with zero outer bias) and by $\phi_2 \in \R^{3(N-M)+1}$ the network comprised of the last $N-M$ hidden neurons.
We write $\Loss_{N-M}$ for the loss defined on networks with $N-M$ hidden neurons.
Then, for any perturbation $\phi_h = \phi+h \in \R^{3N+1}$ of $\phi$ with the same decomposition into its first $M$ and last $N-M$ hidden neurons, we can write $\NNfctPath{h}(x) = \NNfctPath{1,h}(x) + \NNfctPath{2,h}(x)$ and, hence,
\begin{equation*}
	\Loss(\phi_h) = \int_0^1 \NNfctPath{1,h}(x)^2 dx + 2 \int_0^1 \NNfctPath{1,h}(x) ( \NNfctPath{2,h}(x)-x) dx + \Loss_{N-M}(\phi_{2,h}).
\end{equation*}%
Since the first $M$ hidden neurons of $\phi$ are flat degenerate, $\NNfctPath{1,h}(x)$ is given by
\begin{equation*}
	\NNfctPath{1,h}(x) = \ssum{j=1}{M} h_{j+2N} \ReLU{h_jx+h_{j+N}}.
\end{equation*}%
In particular, $\NNfctPath{1,h}(x)/\|h\| \rightarrow 0$ uniformly in $x \in [0,1]$ as $h \rightarrow 0$.
Denote by $\tilde{h}$ the last $3(N-M)$ components of $h$.
Since $\phi_2$ has only non-degenerate neurons, $\Loss_{N-M}$ is differentiable at $\phi_2$ with some gradient $A$.
Using that the first $M$ hidden neurons of $\phi$ do not contribute to its realization and, hence, $\Loss(\phi) = \Loss_{N-M}(\phi_2)$, we find
\begin{equation*}
\begin{split}
	\lim_{h \rightarrow 0} \frac{\Loss(\phi_h)-\Loss(\phi)-A\tilde{h}}{\|h\|} = &\lim_{h \rightarrow 0} \frac{\Loss_{N-M}(\phi_{2,h})-\Loss_{N-M}(\phi_2)-A\tilde{h}}{\|\tilde{h}\|} \frac{\|\tilde{h}\|}{\|h\|} \\
	+ &\lim_{h \rightarrow 0} \frac{1}{\|h\|} \Big( \int_0^1 \NNfctPath{1,h}(x)^2 dx + 2 \int_0^1 \NNfctPath{1,h}(x) ( \NNfctPath{2,h}(x)-x) dx  \Big) = 0.
\end{split}
\end{equation*}%
This proves differentiability of $\Loss$ at $\phi$.
\end{proof}

So far, we have seen that, in some regions of the parameter space, the loss is differentiable while in others it may not be.
In the following, we show that, for type-2-active neurons, one even has twice continuous differentiability.

\begin{lemma}
\label{LemTwiceDiff}
	Let $i,j \in \{1,\dots,N\}$.
If the $i^{th}$ and $j^{th}$ hidden neuron of $\phi \in \R^{3N+1}$ are type-2-active, then $\Loss$ is twice continuously differentiable in $(w_i,w_j,b_i,b_j,v,c)$ in a neighborhood of $\phi$ in $\R^{3N+1}$.
\end{lemma}

\begin{proof}
	Note that we established twice continuous differentiability of $\Loss$ in $(v,c)$ in \cref{LemLossSmooth}.
Suppose the $i^{th}$ and $j^{th}$ hidden neuron of $\phi^0 = (w^0,b^0,v^0,c^0) \in \R^{3N+1}$ are type-2-active.
Since a small perturbation of a type-2-active neuron remains type-2-active and since a type-2-active neuron is non-degenerate, it follows from \cref{LemLossDiff} that $\Loss$ is differentiable in $(w_j,b_j)$ in a neighborhood $U \subseteq \R^{3N+1}$ of $\phi^0$ with partial derivatives
\begin{equation*}
	\frac{\partial}{\partial w_j} \Loss(\phi) = 2 v_j \int_{I_j} x(\NNfct(x)-x) dx \quad \text{and} \quad \frac{\partial}{\partial b_j} \Loss(\phi) = 2 v_j \int_{I_j} (\NNfct(x)-x) dx
\end{equation*}%
for any $\phi = (w,b,v,c) \in U$.
Because the $j^{th}$ hidden neuron is assumed to be type-2-active, the interval $I_j^0$ is exactly $[0,t_j^0]$ or $[t_j^0,1]$ for the breakpoint $t_j^0 = -b_j^0/w_j^0 \in (0,1)$.
Assume $I_j^0 = [0,t_j^0]$ as the other case is dealt with analogously.
By shrinking $U$ if necessary, we therefore integrate over $[0,-b_j/w_j]$ in the above partial derivatives for all $\phi = (w,b,v,c) \in U$.
In particular, the integration boundaries vary smoothly in $(w_j,b_j)$ in $U$.
So, it follows from Leibniz' rule that these partial derivatives are continuously differentiable with respect to $(w_j,b_j)$.
Furthermore, since $t_j = -b_j/w_j$ does not depend on $(w_i,b_i,v,c)$, it follows from dominated convergence that $\partial \Loss(\phi) / \partial w_j$ and $\partial \Loss(\phi) / \partial b_j$ are also differentiable with respect to $(w_i,b_i,v,c)$.
The mixed partial derivative with respect to $w_i$ and $w_j$ is given by
\begin{equation*}
	\frac{\partial}{\partial w_i} \frac{\partial}{\partial w_j} \Loss(\phi) = 2v_j \int_{I_j} x \, \frac{\partial}{\partial w_i} \NNfct(x) dx = 2v_iv_j \int_{I_i \cap I_j} x^2 dx.
\end{equation*}%
That the $i^{th}$ and $j^{th}$ hidden neuron are type-2-active ensures that $\int_{I_i \cap I_j} x^2 dx$ is continuous in $(w_i,w_j,b_i,b_j)$ and, hence, that $\partial^2 \Loss(\phi) / (\partial w_i \partial w_j)$ is continuous in $(w_i,w_j,b_i,b_j,v,c)$.
Analogous considerations show that all mixed partial derivatives with respect to $w_i,w_j,b_i,b_j,v,c$ up to second order exist and are continuous.
Thus, $\Loss$ restricted to $(w_i,w_j,b_i,b_j,v,c)$ is twice continuously differentiable in a neighborhood of $\phi^0$.
\end{proof}

\begin{remark}
	We mentioned in \cref{rem_proper_crit} that all critical points and local extrema of $\Loss$ are actually proper critical points and, hence, the classification actually does not deal with points of non-differentiability.
Furthermore, by modifying the Gradient Descent algorithm and the initialization in an appropriate way, one can ensure that the trajectories of the algorithm avoid any points of non-differentiability; see \cite{Wojtowytsch2020} and also the appendix in \cite{ChizatBach2020}.
Nonetheless, to formally prove the classification, including that all critical points are proper, an extensive regularity analysis of the loss function as done in this section is necessary.
\end{remark}


\subsection{Critical points of the loss function with affine realization}
\label{section_crit_const}

In this and the next section, we develop the building blocks necessary for proving the main result.
The first lemma establishes one direction of the equivalence in \cref{ThrmMain}.\eqref{ThrmMainLocMin}.

\begin{lemma}
\label{LemCritLocMin}
	Suppose $\phi \in \R^{3N+1}$ is centered and all of its hidden neurons satisfy one of the properties \eqref{ThrmMainLocMinInact}-\eqref{ThrmMainLocMinSemiIn2} in \cref{ThrmMain}.
Then $\phi$ is a local minimum of $\Loss$.
\end{lemma}

\begin{proof}
	Denote by $J_0 \subseteq \{1,\dots,N\}$ the set of those hidden neurons of $\phi$ that satisfy \cref{ThrmMain}.\eqref{ThrmMainLocMinSemiIn1}, and, likewise, denote by $J_1 \subseteq \{1,\dots,N\}$ the set of those hidden neurons of $\phi$ that satisfy \cref{ThrmMain}.\eqref{ThrmMainLocMinSemiIn2}.
Write $\phi = (w^0,b^0,v^0,c^0)$ and consider $\psi = (w,b,v,c) \in U$ in a small neighborhood $U$ of $\phi$.
Since a small perturbation of an inactive neuron remains inactive, we have for all $\psi \in U$ and every $x \in [0,1]$ that
\begin{equation*}
	\NNfctAlt(x) = c + \ssum{j \in J_0 \cup J_1}{} v_j \ReLU{w_j x+b_j}
\end{equation*}%
if $U$ is small enough.
Moreover, for any $j \in J_0$ and $\psi \in U$, note that $\ReLU{w_jx+b_j} = 0$ for all $x \in [1/4,1]$.
Similarly, $\ReLU{w_jx+b_j} = 0$ for all $x \in [0,3/4]$ if $j \in J_1$.
Since we also know $v_j^0 > 0$ for all $j \in J_0$ and $v_j^0 < 0$ for all $j \in J_1$, we find that the realization of $\psi \in U$ satisfies
\begin{equation*}
	\NNfctAlt(x) =
	\begin{cases}
		c + \ssum{j \in J_0}{} v_j \ReLU{w_j x+b_j} \geq c &\text{if } x \in [0,1/4] \\
		c &\text{if } x \in [1/4,3/4] \\
		c + \ssum{j \in J_1}{} v_j \ReLU{w_j x+b_j} \leq c &\text{if } x \in [3/4,1]
	\end{cases}
\end{equation*}%
for sufficiently small $U$.
In particular, it follows that $|\NNfctAlt(x)-x| \geq |c-x|$ for all $x \in [0,1]$ and, because $\phi$ is centered, that
\begin{equation*}
	\Loss(\psi) \geq \int_{0}^{1} (c-x)^2 dx \geq \int_{0}^{1} (\tfrac{1}{2}-x)^2 dx = \Loss(\phi).
\end{equation*}%
Thus, $\phi$ is a local minimum.
\end{proof}

The proof of the next lemma revolves, for the most part, around the argument \eqref{ExampleArgumentMethod}, presented in \cref{section_methods}.
The last statement of the lemma paired with \cref{LemLossContDiff} shows that saddle points with an affine realization are also points of differentiability of $\Loss$.

\begin{lemma}
\label{LemCritAff}
	Suppose $\phi \in \R^{3N+1}$ is a critical point or a local extremum of $\Loss$ but not a global minimum and that $\NNfct$ is affine on $[0,1]$.
Then $\phi$ is centered and does not have any active or non-flat semi-active neurons, so, in particular, $\NNfct \equiv 1/2$.
Moreover, if $\phi$ is a saddle point, then it also does not have any non-flat degenerate neurons.
\end{lemma}

\begin{proof}
	We know from \cref{LemLossDiff} that $\Loss$ is differentiable in those coordinates that correspond to non-degenerate neurons and its partial derivatives must vanish at $\phi$.
Thus, the argument using \eqref{ExampleArgumentMethod} shows that $\phi$ does not have any type-1-active or non-flat semi-active neurons.
If $\phi$ had a non-flat type-2-active neuron, say the $j^{th}$, then we could, using the same argument with $I_j$ in place of $[0,1]$, conclude that $\NNfct(x) = x$ on $I_j$.
But since $\NNfct$ was assumed to be affine, this could only be true if $\phi$ were a global minimum.
Having no type-1-active or non-flat type-2-active neurons, $\NNfct$ must be constant.
By the fourth equation of \eqref{LossGradient}, this constant is $1/2$, so $\phi$ is centered.

Next, suppose that the $j^{th}$ hidden neuron is flat type-2-active.
In particular, $I_j = [0,t_j]$ or $I_j = [t_j,1]$, where $t_j = -b_j/w_j \in (0,1)$ is the breakpoint.
After dividing by $2w_j$, the integral in the third equation of \eqref{LossGradient} evaluates to
\begin{equation*}
	0 = \int_{I_j} (x-t_j)(\tfrac{1}{2}-x) dx =
	\begin{cases}
		-\tfrac{1}{6} t_j^2 (\frac{3}{2} - t_j) &\text{if } I_j = [0,t_j] \\
		-\tfrac{1}{6} (1-t_j)^2(t_j+\frac{1}{2}) &\text{if } I_j = [t_j,1]
	\end{cases} \Bigg\} \ne 0,
\end{equation*}%
yielding a contradiction.
Lastly, suppose $\phi$ is a saddle point.
If there were a non-flat degenerate neuron, then $\Grad(\phi) = 0$ would imply $0 = \int_0^1 x(\NNfct(x)-x)dx$.
But since we know that $\NNfct(x) \equiv 1/2$, this cannot be.
\end{proof}

The next lemma serves as the basis of \cref{ThrmMain}.\eqref{ThrmMainSaddleTriv}.
However, note that we also consider the possibility of a non-flat degenerate neuron, whereas \cref{ThrmMain}.\eqref{ThrmMainSaddleDeg} requires the degenerate neuron to be flat.
This generalization is needed in the proof of \cref{ThrmMain}.\eqref{ThrmMainLocMin}, which will be given later by way of contradiction.
In addition, \cref{LemCritSaddleT} shows that non-global local minima with an affine realization cannot have non-flat degenerate neurons and, hence, are points of differentiability of $\Loss$ by \cref{LemLossContDiff}.
Together with the preceding lemma and \cref{LemDiffGlobalMinima,LemLossContDiff}, we conclude that all critical points and local extrema with an affine realization are points of differentiability.

\begin{lemma}
\label{LemCritSaddleT}
	Suppose $\phi \in \R^{3N+1}$ is a critical point or a local extremum of $\Loss$ but not a global minimum and that $\NNfct$ is affine on $[0,1]$.
Suppose further that at least one of its hidden neurons satisfies one of the properties \eqref{ThrmMainSaddleSemiAc}-\eqref{ThrmMainSaddleSemiIn2} in \cref{ThrmMain} or is degenerate.
Then $\phi$ is a saddle point.
\end{lemma}

\begin{proof}
	Since, by \cref{LemNoMaxima}, $\Loss$ cannot have any local maxima, it is enough to show that $\Loss$ is strictly decreasing along some direction starting from $\phi$.
First, assume that the $j^{th}$ hidden neuron of $\phi$ is flat semi-active.
Then \cref{LemLossSmooth} asserts smoothness of the loss in the coordinates of the $j^{th}$ hidden neuron and
\begin{equation*}
\begin{split}
	\frac{\partial}{\partial w_j} \frac{\partial}{\partial w_j} \Loss(\phi) &= 2 v_j \int_{0}^{1} x \, \frac{\partial}{\partial w_j} \NNfct(x) dx = 0, \\
	\frac{\partial}{\partial v_j} \frac{\partial}{\partial w_j} \Loss(\phi) &= 2 v_j \int_{0}^{1} x \, \frac{\partial}{\partial v_j} \NNfct(x) dx + 2 \int_{0}^{1} x (\NNfct(x)-x) dx \\
	&= 2 \int_{0}^{1} x (\NNfct(x)-x) dx =: R, \\
	\frac{\partial}{\partial v_j} \frac{\partial}{\partial v_j} \Loss(\phi) &= 2 \int_{0}^{1} (w_jx+b_j) \, \frac{\partial}{\partial v_j} \NNfct(x) dx =: S,
\end{split}
\end{equation*}%
where we used that the $j^{th}$ hidden neuron is flat.
Since $2\int_{0}^{1} (\NNfct(x)-x) dx = \frac{\partial}{\partial c} \Loss(\phi) = 0$, we must have $R \ne 0$ for otherwise $\phi$ would be a global minimum by the argument \eqref{ExampleArgumentMethod}.
This yields
\begin{equation*}
	\det
	\begin{pmatrix}
		\frac{\partial}{\partial w_j} \frac{\partial}{\partial w_j} \Loss(\phi) & \frac{\partial}{\partial w_j} \frac{\partial}{\partial v_j} \Loss(\phi) \\
		\frac{\partial}{\partial v_j} \frac{\partial}{\partial w_j} \Loss(\phi) & \frac{\partial}{\partial v_j} \frac{\partial}{\partial v_j} \Loss(\phi)
	\end{pmatrix}
	= \det
	\begin{pmatrix}
		0 & R \\
		R & S
	\end{pmatrix}
	= -R^2 < 0.
\end{equation*}%
In particular, this matrix must have a strictly negative eigenvalue, and a second order expansion of the loss restricted to $(w_j,v_j)$ shows that $\Loss$ is strictly decreasing along the direction of an eigenvector associated to this negative eigenvalue.

Next, assume that the $j^{th}$ hidden neuron is semi-inactive with $I_j = \{0\}$ and $v_j \leq 0$ (case one) or that it is degenerate with $v_j \leq 0$ (case two).
In either case, note that $b_j = 0$ and consider the perturbation $\phi_s = (w^s,b^s,v^s,c^s)$, $s \in [0,1]$, of $\phi = \phi_0$ given by $w_j^s = w_j-s$, $b_j^s = -sw_j^s$, and $v_j^s = v_j-s$ (all other coordinates coincide with those of $\phi$).
Note that we have $w_j^s < 0$ and $v_j^s < 0$ for all $s \in (0,1]$ in both cases.
For simplicity, denote $a^s = v_j^sw_j^s$.
By \cref{LemCritAff}, we already know that $\phi$ is centered and does not have any active or non-flat semi-active neurons.
Thus, for every $s,x \in [0,1]$, we can write
\begin{equation*}
	\NNfctPath{s}(x) = c + v_j^s \ReLU{w_j^sx+b_j^s} = c + v_j^s \ReLU{w_j^s(x-s)} = \tfrac{1}{2} + a^s(x-s)\IndFct{[0,s]}(x).
\end{equation*}%
Using this formula, we have for all $s \in [0,1]$
\begin{equation*}
\begin{split}
	\Loss(\phi_s) - \Loss(\phi) &= \int_{0}^{s} [a^s(x-s)]^2 dx - \int_{0}^{s} 2a^s(x-s) (x-\tfrac{1}{2}) dx \\
	&= \tfrac{1}{3} a^s(a^s+1) s^3 - \tfrac{1}{2} a^s s^2 \\
	&=
	\begin{cases}
		-\tfrac{1}{2} v_jw_j s^2 + \mathcal{O}(s^3) &\text{if } w_j \ne 0 \ne v_j \\
	-\tfrac{1}{2}|v_j+w_j| s^3 + \mathcal{O}(s^4) &\text{if } w_j \ne 0 = v_j \text{ or } w_j = 0 \ne v_j \\
	-\tfrac{1}{2}s^4 + \mathcal{O}(s^5) &\text{if } w_j = 0 = v_j,	
	\end{cases}
\end{split}
\end{equation*}%
which is strictly negative for small $s>0$.
Hence, $\phi$ is a saddle point.

Lastly, assume that the $j^{th}$ hidden neuron is semi-inactive with $I_j = \{1\}$ and $v_j \geq 0$ (case one) or that it is degenerate with $v_j > 0$ (case two).
This is dealt with the same way as the previous step.
Let $\phi_s \in \R^{3N+1}$, $s \in [0,1]$, be given by $w_j^s = w_j+s$, $b_j^s = -(1-s)w_j^s$, and $v_j^s = v_j+s$.
This time, we have $w_j^s > 0$ and $a^s = v_j^sw_j^s > 0$ for all $s \in (0,1]$ in both cases.
The realization of $\phi_s$ on $[0,1]$ is given for all $s,x \in [0,1]$ by
\begin{equation*}
	\NNfctPath{s}(x) = c + v_j^s \ReLU{w_j^sx+b_j^s} = \tfrac{1}{2} + a^s(x-1+s)\IndFct{[1-s,1]}(x).
\end{equation*}%
Essentially by the same computation as in the previous step,
\begin{equation*}
\begin{split}
	\Loss(\phi_s) - \Loss(\phi) &= \tfrac{1}{3} a^s(a^s+1) s^3 - \tfrac{1}{2} a^s s^2 \\
	&=
	\begin{cases}
		-\tfrac{1}{2} v_jw_j s^2 + \mathcal{O}(s^3) &\text{if } w_j \ne 0 \ne v_j \\
	-\tfrac{1}{2}(v_j+w_j) s^3 + \mathcal{O}(s^4) &\text{if } w_j \ne 0 = v_j \text{ or } w_j = 0 \ne v_j \\
	-\tfrac{1}{2}s^4 + \mathcal{O}(s^5) &\text{if } w_j = 0 = v_j,	
	\end{cases}
\end{split}
\end{equation*}%
from which we conclude that $\phi$ is a saddle point.
\end{proof}

This finishes the treatment of the affine case, and we now tend to the more involved non-affine case in the next section.


\subsection{Critical points of the loss function with non-affine realization}
\label{section_crit_non_const}

The following lemma is the main tool for this section.
It generalizes the argument \eqref{ExampleArgumentMethod} that we presented in \cref{section_methods}; see \cref{LemUniqueAffGeneral}.\eqref{LemUniqueAffGeneralConsequNzero} below.
This lemma captures the combinatorics of piecewise affine functions satisfying conditions of the form \eqref{CritCharacteristic}.

\begin{lemma}
\label{LemUniqueAffGeneral}
	Let $n \in \N_0$, $A_0,\dots,A_n,B_0,\dots,B_n,q_0,\dots,q_{n+1} \in \R$ satisfy $q_0 < \dots < q_{n+1}$, and consider a function $f \in C([q_0,q_{n+1}],\R)$ satisfying for all $i \in \{0,\dots,n\}$, $x \in [q_i,q_{i+1}]$ that $f(x) = A_ix+B_i$ and $\int_{q_i}^{q_{i+1}} (f(y)-y) dy = 0$.
Then
\begin{enumerate}[\rm (i)]\itemsep = 0em

\item we have for all $i \in \{0,\dots,n\}$ that
\begin{equation}
\label{LemUniqueAffGeneralFormulas}
\begin{split}
	A_i - 1 &= (-1)^i \frac{q_1-q_0}{q_{i+1}-q_i} (A_0-1), \\
	B_i &= (-1)^{i+1} \frac{q_{i+1}+q_i}{2} \frac{q_1-q_0}{q_{i+1}-q_i} (A_0-1),
\end{split}
\end{equation}%

\item\label{LemUniqueAffGeneralConsequ1} we have $f = \ID{[q_0,q_{n+1}]}$ $\iff$ $\forall i \in \{0,\dots,n\} \colon A_i = 1, ~B_i = 0$ \newline
$\phantom{\text{we have}}$ $\iff$ $\exists i \in \{0,\dots,n\} \colon A_i = 1, ~B_i = 0$ $\iff$ $\exists i \in \{0,\dots,n\} \colon f|_{[q_i,q_{i+1}]} = \ID{[q_i,q_{i+1}]}$,

\item for all $i \in \{0,\dots,n\}$ we have $\mathrm{sign}(A_i-1) = (-1)^i \mathrm{sign}(A_0-1)$.

\end{enumerate}
If, in addition, $0 = \int_{q_0}^{q_{n+1}} x(f(x)-x) dx$, then
\begin{enumerate}[\rm (i)]\itemsep = 0em
\setcounter{enumi}{3}

\item we have $0 = (A_0-1) \sum_{i=0}^{n} (-1)^i (q_{i+1}-q_i)^2$,

\item\label{LemUniqueAffGeneralConsequSum} if $f \ne \ID{[q_0,q_{n+1}]}$, then $0 = \ssum{i=0}{n} (-1)^{i+1} (q_{i+1}-q_i)^2$,

\item\label{LemUniqueAffGeneralConsequNzero} if $n=0$, then $f = \ID{[q_0,q_1]}$.

\end{enumerate}
\end{lemma}

\begin{proof}
	First note that we must have $A_iq_{i+1} + B_i = A_{i+1}q_{i+1} + B_{i+1}$ for all $i \in \{0,\dots,n-1\}$.
Moreover, the assumption $0 = \int_{q_i}^{q_{i+1}} (f(x)-x) dx$ is equivalent to $B_i = - \frac{1}{2} (q_{i+1}+q_i) (A_i-1)$.
Combining these yields
\begin{equation*}
	A_{i+1} - 1 = - \frac{q_{i+1}-q_i}{q_{i+2}-q_{i+1}} (A_i-1)
\end{equation*}%
for all $i \in \{0,\dots,n-1\}$.
Induction then proves the formula for $A_i-1$, and the formula for $B_i$ follows.
Lastly, by plugging the formulas for $A_i$ and $B_i$ into $f(x)$, we compute
\begin{equation*}
	\int_{q_0}^{q_{n+1}} x(f(x)-x) dx = \sum_{i=0}^{n} \int_{q_i}^{q_{i+1}} x((A_i-1)x+B_i) dx = \frac{q_1-q_0}{12} (A_0-1) \sum_{i=0}^{n} (-1)^i (q_{i+1}-q_i)^2.
\end{equation*}%
The remaining items follow immediately.
\end{proof}

In order to apply this lemma later on, let us verify that our network always satisfies the condition $\int_{q_i}^{q_{i+1}} (f(y)-y) dx = 0$ for suitable choices of $q_i$ and $q_{i+1}$.

\begin{lemma}
\label{LemZeroInt}
	Suppose $\phi \in \R^{3N+1}$ is a critical point or a local extremum of $\Loss$ and denote by $0 = q_0 < q_1 < \dots < q_n < q_{n+1} = 1$, for $n \in \N_0$, the roughest partition such that $\NNfct$ is affine on all subintervals $[q_i,q_{i+1}]$.
Then we have for all $i \in \{0,\dots,n\}$ that
\begin{equation*}
	\int_{q_i}^{q_{i+1}} (\NNfct(x)-x) dx = 0.
\end{equation*}%
\end{lemma}

\begin{proof}
	First, note that $\phi$ must have a non-flat type-2-active neuron whose breakpoint is $q_i$, for all $i \in \{1,\dots,n\}$.
From the fourth line of \eqref{LossGradient}, we know that $\int_0^1 (\NNfct(x)-x) dx = 0$.
This and the second line of \eqref{LossGradient} imply, for any non-flat type-2-active neuron $j$,
\begin{equation*}
	\int_{I_j} (\NNfct(x)-x) dx = 0 = \int_{[0,1] \backslash I_j} (\NNfct(x)-x) dx.
\end{equation*}%
Since either $I_j = [0,t_j]$ or $[0,1] \backslash I_j = [0,t_j]$, it follows that $\int_0^{q_i} (\NNfct(x)-x) dx = 0$, for all $i \in \{0,\dots,n+1\}$.
Taking differences of these integrals yields the desired statement.
\end{proof}

Next, as a first application of \cref{LemUniqueAffGeneral}, we prove that only global minima can have type-1-active or non-flat semi-active neurons.
We already established this in \cref{LemCritAff} in the affine case, but now we extend it to the non-affine case.
The statement from \cref{LemCritAff} about saddle points not having non-flat degenerate neurons also holds in the non-affine case, but we will not see this until later in \cref{section_main_proof_ID_target}.

\begin{lemma}
\label{LemCritGen}
	Suppose $\phi \in \R^{3N+1}$ is a critical point or a local extremum of $\Loss$ but not a global minimum.
Then $\phi$ does not have any type-1-active or non-flat semi-active neurons.
\end{lemma}

\begin{proof}
	For affine $\NNfct$, the result has been established in \cref{LemCritAff}.
Thus, suppose $\NNfct$ is not affine on $[0,1]$ and that $\phi$ has a type-1-active or non-flat semi-active neuron.
Denote by $0 = q_0 < q_1 < \dots < q_n < q_{n+1} = 1$, for $n \in \N$, the roughest partition such that $\NNfct$ is affine on all subintervals $[q_i,q_{i+1}]$.
We know from \cref{LemZeroInt} that $\int_{q_0}^{q_1} (\NNfct(x)-x) dx = 0$, and we claim that also $\int_{q_0}^{q_1} x(\NNfct(x)-x) dx = 0$.
To prove this, note that $\phi$ must have at least one non-flat type-2-active neuron (without loss of generality the first) with breakpoint $-b_1/w_1 = q_1$.
Moreover, \eqref{CritCharacteristic} shows that $0 = \int_{0}^{1} x(\NNfct(x)-x) dx$ if applied with the type-1-active or non-flat semi-active neuron.
Using this and $\frac{\partial}{\partial w_1} \Loss(\phi) = 0$, one deduces the claim as in the proof of \cref{LemZeroInt}.
Hence, we conclude $\NNfct|_{[q_0,q_1]} = \ID{[q_0,q_1]}$ with the argument \eqref{ExampleArgumentMethod}.
But then we also get $\NNfct = \ID{[q_0,q_{n+1}]}$ by \cref{LemUniqueAffGeneral}.\eqref{LemUniqueAffGeneralConsequ1} and \cref{LemZeroInt}, yielding a contradiction.
\end{proof}

We now turn to the proof of \cref{ThrmMain}.\eqref{ThrmMainSaddleNonTriv}.
More precisely, we show that critical points and local extrema whose realizations are not affine must take a very specific form.
The only degree of freedom of their realization functions is a single parameter varying over the set of even integers in $\{1,\dots,N\}$.
Examples of the possible realizations are shown in Figure \ref{fig_plot_realizations}, which illustrates that the degree of freedom is reflected by the number of breakpoints.
Once this number is fixed, the shape of the function is uniquely determined: the breakpoints are equally spaced in the interval $[0,1]$, and the slope of the realization on each affine segment alternates between two given values in such a way that the function symmetrically oscillates around the diagonal.
In addition, we deduce in \cref{LemCritNonAff} that critical points and local extrema can realize these functions only in a very specific way, limited by few combinatorial choices.

\begin{lemma}
\label{LemCritNonAff}
	Suppose $\phi \in \R^{3N+1}$ is a critical point or a local extremum of $\Loss$ but not a global minimum and that $\NNfct$ is not affine on $[0,1]$.
Denote by $0 = q_0 < q_1 < \dots < q_n < q_{n+1} = 1$, for $n \in \N$, the roughest partition such that $\NNfct$ is affine on all subintervals $[q_i,q_{i+1}]$, and denote by $K_i \subseteq \{1,\dots,N\}$ the set of all type-2-active neurons of $\phi$ whose breakpoint is $q_i$.
Then the following hold:
\begin{enumerate}[\rm (i)]\itemsep = 0em

\item $n$ is even,

\item $q_i = \frac{i}{n+1}$ for all $i \in \{1,\dots,n\}$,

\item $-b_j/w_j \in \{q_1,\dots,q_n\}$ for all type-2-active neurons $j \in \{1,\dots,N\}$ of $\phi$,

\item $\mathrm{sign}(w_j) = (-1)^{i+1}$ for all $i \in \{1,\dots,n\}$, $j \in K_i$,

\item $\ssum{j \in K_i}{} v_j w_j = 2/(n+1)$ for all $i \in \{1,\dots,n\}$,

\item $\phi$ is centered,

\item\label{LemCritNonAffFct} $\NNfct(x) = x - \frac{(-1)^i}{n+1} \big( x - \frac{i+1/2}{n+1} \big)$ for all $i \in \{0,\dots,n\}$, $x \in [q_i,q_{i+1}]$.

\end{enumerate}
\end{lemma}

The proof of this lemma requires a successive application of \cref{LemUniqueAffGeneral}.
We prove the statements of the lemma in a different order than stated.
First of all, \cref{LemUniqueAffGeneral}.\eqref{LemUniqueAffGeneralConsequ1} will enforce the correct sign for each $w_j$, $j \in K_i$.
That $n$ is even will be a consequence of these signs.
It will also follow from the signs together with \cref{LemUniqueAffGeneral}.\eqref{LemUniqueAffGeneralConsequSum} that $q_i = \frac{i}{n+1}$.
Afterwards, we use the formulas \eqref{LemUniqueAffGeneralFormulas} from \cref{LemUniqueAffGeneral} to verify that any type-2-active neuron must have as breakpoint one of $q_1,\dots,q_n$.
Once this has been shown, we obtain a more explicit version of those formulas and deduce $\ssum{k \in K_i}{} v_k w_k = 2/(n+1)$.
That $\NNfct$ takes exactly the form in \cref{LemCritNonAff}.\eqref{LemCritNonAffFct} is a byproduct of the last derivation, and that $\phi$ is centered is shown last.

\begin{figure}
	\centering
	\includegraphics[width=\linewidth , trim = {4.5cm 0cm 3.5cm 0cm} , clip]{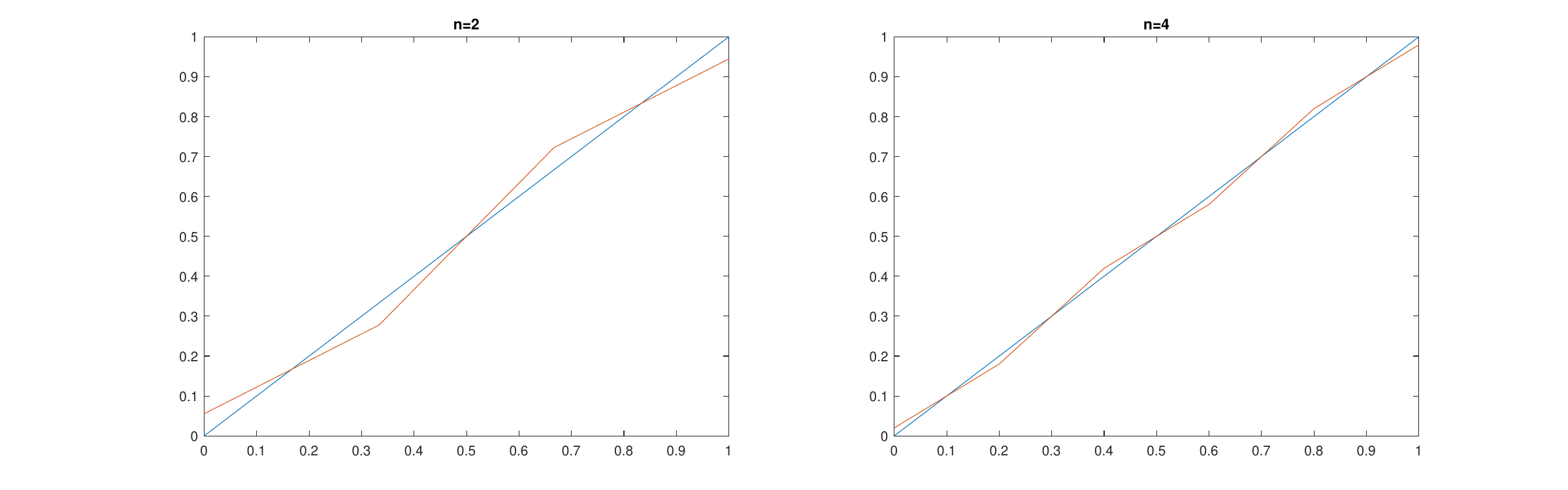}
	\caption{Examples of the network realizations (red) in \cref{LemCritNonAff} for the cases $n=2$ and $n=4$. The blue line is the target function (identity function).}
	\label{fig_plot_realizations}
\end{figure}

\begin{proof}[Proof of \cref{LemCritNonAff}]
	We begin by noting that none of the sets $K_i$, $i \in \{1,\dots,n\}$, can be empty.
Furthermore, the third equation of \eqref{LossGradient} and \cref{LemZeroInt} imply that \eqref{CritCharacteristic} holds for all neurons in $\bigcup_i K_i$ even if they are flat.
Applying \cref{LemUniqueAffGeneral}.\eqref{LemUniqueAffGeneralConsequ1}, which we can do by \cref{LemZeroInt}, ensures that $\NNfct|_{[q_i,q_{i+1}]} \ne \ID{[q_i,q_{i+1}]}$ for all $i \in \{0,\dots,n\}$.
In particular, \eqref{CritCharacteristic} and the argument \eqref{ExampleArgumentMethod} show for all $i \in \{1,\dots,n-1\}$ and $j_0 \in K_i$, $j_1 \in K_{i+1}$ that $\mathrm{sign}(w_{j_0}) \ne \mathrm{sign}(w_{j_1})$ for otherwise we would have $I_{j_0} \backslash I_{j_1} = [q_i,q_{i+1}]$ or $I_{j_1} \backslash I_{j_0} = [q_i,q_{i+1}]$ (depending on the sign) and, hence,
\begin{equation*}
	\int_{q_i}^{q_{i+1}} (\NNfct(x)-x) dx = 0 = \int_{q_i}^{q_{i+1}} x(\NNfct(x)-x) dx.
\end{equation*}%
Likewise, we must have $\int_{0}^{q_1} x(\NNfct(x)-x) dx \ne 0$ and, hence, $w_j > 0$ for any $j \in K_1$.
Combining the previous two arguments establishes $\mathrm{sign}(w_j) = (-1)^{i+1}$ for any $i \in \{1,\dots,n\}$, $j \in K_i$.
Just like $w_j > 0$ for any $j \in K_1$, we must also have $w_j < 0$ for any $j \in K_n$.
Thus, $-1 = \mathrm{sign}(w_j) = (-1)^{n+1}$ for all $j \in K_n$, so $n$ is even.
Now that we know the sign of each parameter $w_j$ for neurons $j \in \bigcup_i K_i$, we can use \eqref{CritCharacteristic} again to find that $\int_{q_i}^{q_{i+2}} x(\NNfct(x)-x) dx = 0$ for all $i \in \{0,\dots,n-1\}$.
Then \cref{LemUniqueAffGeneral}.\eqref{LemUniqueAffGeneralConsequSum} (with the partition $q_i,q_{i+1},q_{i+2}$) tells us
\begin{equation*}
	0 = (q_{i+2}-q_{i+1})^2 - (q_{i+1}-q_i)^2.
\end{equation*}%
This can only hold for all $i \in \{0,\dots,n-1\}$ if the points $q_1,\dots,q_n$ are equidistributed, which means $q_i = i/(n+1)$.
Next, if we denote $\NNfct(x) = A_ix+B_i$ on $[q_i,q_{i+1}]$, then the formulas \eqref{LemUniqueAffGeneralFormulas} must hold for all $i \in \{0,\dots,n\}$.
Since $q_1,\dots,q_n$ are equidistributed, the formulas simplify to 
\begin{equation}
\label{LemCritNonAffPfFormulas}
	A_i-1 = (-1)^i(A_0-1) \quad \text{and} \quad B_i = (-1)^{i+1}\frac{i+\frac{1}{2}}{n+1}(A_0-1)
\end{equation}%
for all $i \in \{0,\dots,n\}$.
Using \eqref{LemCritNonAffPfFormulas}, one can verify that any type-2-active neuron of $\phi$ must have as breakpoint one of the points $q_1,\dots,q_n$.
If this were not the case, say the $j^{th}$ hidden neuron were type-2-active with breakpoint $t_j = -b_j/w_j$, then one could choose $i \in \{0,\dots,n\}$ such that $q_i < t_j < q_{i+1}$.
Using \eqref{CritCharacteristic}, \eqref{LemCritNonAffPfFormulas}, and \cref{LemZeroInt}, the integral from the third line of \eqref{LossGradient} reads (after dividing by $2w_j$)
\begin{equation*}
\begin{split}
	\int_{I_j} (x-t_j)(\NNfct(x)-x) dx &= \int_{[q_i,q_{i+1}] \cap I_j} (x-t_j)(\NNfct(x)-x) dx -
	\begin{cases}
		0 &\text{if } i \text{ is even} \\
		\int_{q_i}^{q_{i+1}} x(\NNfct(x)-x) dx &\text{if } i \text{ is odd}
	\end{cases} \\
	&=
	\begin{cases}
		\frac{1}{6} (A_0-1) (t_j-q_i)^2 (q_{i+1} - t_j + \frac{1}{2(n+1)})
		&\begin{array}{l}
			\text{if } I_j = [0,t_j] \text{ and } i \text{ is even}  \\
			\text{or if } I_j = [t_j,1] \text{ and } i \text{ is odd}
		\end{array} \\
		\frac{1}{6} (A_0-1) (q_{i+1}-t_j)^2 (t_j - q_i + \frac{1}{2(n+1)})
		&\begin{array}{l}
			\text{if } I_j = [0,t_j] \text{ and } i \text{ is odd} \\
			\text{or if } I_j = [t_j,1] \text{ and } i \text{ is even}.
		\end{array}
	\end{cases}
\end{split}
\end{equation*}%
So, the partial derivative of $\Loss$ with respect to $v_j$ does not vanish, yielding a contradiction.
This proves that all type-2-active neurons lie in $\bigcup_i K_i$.
In particular, we can write
\begin{equation*}
	A_l = \ssum{\substack{i=1 \\ i \text{ odd}}}{l} \ssum{j \in K_i}{} v_jw_j + \ssum{\substack{i=l+1 \\ i \text{ even}}}{n} \ssum{j \in K_i}{} v_jw_j
\end{equation*}%
for all $l \in \{0,\dots,n\}$ because $\phi$ does not have any type-1-active neurons by \cref{LemCritGen}.
We can combine this formula with \eqref{LemCritNonAffPfFormulas} to find for all $i \in \{0,\dots,n-1\}$
\begin{equation*}
	-(A_0-1) = (-1)^i (A_{i+1}-1) = (-1)^i (A_i-1) + \ssum{j \in K_{i+1}}{} v_jw_j = A_0 - 1 + \ssum{j \in K_{i+1}}{} v_jw_j.
\end{equation*}%
Thus, the quantity $a := \ssum{j \in K_i}{} v_jw_j$ is independent of $i \in \{1,\dots,n\}$.
Consequently, we obtain $A_i = an/2$ for even $i$ (including $i=0$) and $A_i = a(1+n/2)$ for odd $i$.
The identity $A_1-1 = 1-A_0$ then forces $a = 2/(n+1)$.
That $\phi$ has to be centered follows from $\NNfct(0) = B_0$.
\end{proof}

As our final building block for the proof of \cref{ThrmMain}, we show that the networks from \cref{LemCritNonAff} are saddle points of the loss function.
To achieve this, we will find a set of coordinates in which $\Loss$ is twice differentiable and calculate the determinant of the Hessian of $\Loss$ restricted to these coordinates.
It will turn out to be strictly negative, from which it follows that we deal with a saddle point.

\begin{lemma}
\label{LemCritNonAffSaddle}
	Suppose $\phi \in \R^{3N+1}$ is a critical point or a local extremum of $\Loss$ but not a global minimum and that $\NNfct$ is not affine on $[0,1]$.
Then $\phi$ is a saddle point of $\Loss$.
\end{lemma}

\begin{proof}
	Take $n \in \N$ satisfying the assumptions of \cref{LemCritNonAff} and let $K_1 \subseteq \{1,\dots,N\}$ denote the set of those type-2-active neurons with breakpoint $1/(n+1)$.
Denote by $K_1^- \subseteq K_1$ the set of all those hidden neurons $j \in K_1$ with $v_j < 0$.
It may happen that $K_1^-$ is empty.
However, the complement $K_1 \backslash K_1^-$ is never empty since $\ssum{j \in K_1}{} v_j w_j = 2/(n+1)$ and $\mathrm{sign}(w_j) = 1$ for all $j \in K_1$ by \cref{LemCritNonAff}.
Let $j_1 \in K_1$ be any hidden neuron with $v_{j_1} > 0$ and denote by $j_2,\dots,j_l$, for $l \in \{1,\dots,N\}$, an enumeration of $K_1^-$.
Moreover, let $k \in \{1,\dots,N\}$ be any type-2-active neuron with breakpoint $t_k = 2/(n+1)$.

We know from \cref{LemTwiceDiff} that $\Loss$ is twice continuously differentiable in the coordinates of type-2-active neurons and in $(v,c)$.
We will show that the Hessian $H$ of $\Loss$ restricted to $(b_{j_1},\dots,b_{j_l},v_k,c)$ has a strictly negative determinant.

In order to compute this determinant, we introduce some shorthand notation.
For $i \in \{1,\dots,l\}$, denote $\lambda_i = \frac{n+1}{2}v_{j_i}w_{j_i}$ so that $\ssum{i=1}{l} \lambda_i \leq 1$ by the choice of neurons in the collection $\{j_1,\dots,j_l\}$.
Define $\mu = \frac{n+1}{2n}$ and the vectors $u_1 = (v_{j_1},\dots,v_{j_l})$, $u_2 = (\frac{-1}{4n^2\mu}w_k,1)$, and $u = (u_1,u_2)$.
Furthermore, let $D$ be the diagonal matrix with entries $-v_{j_i}^2 / (4 \lambda_i n)$, $i \in \{1,\dots,l\}$, let $A$ be the Hessian of $\Loss$ restricted to $(v_k,c)$, let $B = \mu A - u_2 u_2^T$, and let $E$ be the diagonal block matrix with blocks $D$ and $B$.
Then $H = \frac{1}{\mu}(E+uu^T)$ and, hence,
\begin{equation*}
	\det(H) = \mu^{-(l+2)}(1+u^TE^{-1}u) \det(E)
\end{equation*}%
once we verified that $E$ is invertible.
We calculate directly
\begin{equation*}
	\det(A) = \det
	\begin{pmatrix}
		\frac{2}{3(n\mu)^3} w_k^2  & \frac{-1}{(n\mu)^2} w_k \\
		\frac{-1}{(n\mu)^2} w_k & \frac{n+1}{n\mu}
	\end{pmatrix}
	= \frac{2n-1}{3(n\mu)^4} w_k^2 > 0.
\end{equation*}%
Next, we compute
\begin{equation}
\label{LemCritNonAffSaddleProofGamma}
	\Gamma := \frac{1}{\mu} u_2^T A^{-1} u_2 = \frac{32n^2-21n+3}{16n(2n-1)} \in (0,1).
\end{equation}%
Using $\Gamma$, we obtain $\det(B) = \mu^2(1-\Gamma)\det(A) > 0$ and $B^{-1} = \frac{1}{\mu} A^{-1} + \frac{1}{\mu^2(1-\Gamma)} A^{-1} u_2 u_2^T A^{-1}$.
In particular, $E$ is invertible.
Using $u_2^T B^{-1} u_2 = \frac{\Gamma}{1-\Gamma}$, we can write
\begin{equation*}
	u^TE^{-1}u = u_1^T D^{-1} u_1 + u_2^T B^{-1} u_2 = - 4n \ssum{i=1}{l} \lambda_i + \frac{\Gamma}{1-\Gamma}.
\end{equation*}%
The determinant of $D$ is $-(4n)^{-l} \prod_{i=1}^{l} v_{j_i}^2 |\lambda_i|^{-1} < 0$ so that
\begin{equation*}
	\Delta := -\mu^{-(l+2)}(1-\Gamma)^{-1}\det(D)\det(B)
\end{equation*}%
is strictly positive.
Summing up, we obtain that the determinant of $H$ is
\begin{equation*}
	\det(H) = \Delta \Big( 4n(1-\Gamma) \ssum{i=1}{l} \lambda_i - 1 \Big).
\end{equation*}%
We already mentioned that $\ssum{i=1}{l} \lambda_i \leq 1$.
Finally, we compute $4n(1-\Gamma) = \frac{5n-3}{8n-4} < 1$ to conclude $\det(H) < 0$, which finishes the proof.
\end{proof}

We now have constructed all the tools needed to prove \cref{ThrmMain} in the special case in which the target function is the identity on $[0,1]$.
This will be done in the next section.


\subsection{Classification of the critical points if the target function is the identity}
\label{section_main_proof_ID_target}

In this section, we gather the results of the previous two sections to prove the main theorem in the case where the target function is the identity on $[0,1]$.

\begin{proposition}
\label{PropIDtarget}
	Let $\phi = (w,b,v,c) \in \R^{3N+1}$.
Then the following hold:
\begin{enumerate}[\rm (I)]\itemsep = 0em

\item\label{PropMainLocMax} $\phi$ is not a local maximum of $\Loss$.

\item\label{PropMainDiff} If $\phi$ is a critical point or a local extremum of $\Loss$, then $\Loss$ is differentiable at $\phi$ with gradient $\nabla \Loss (\phi) = 0$.

\item\label{PropMainLocMin} $\phi$ is a non-global local minimum of $\Loss$ if and only if $\phi$ is centered and, for all $j \in \{1,\dots,N\}$, the $j^{th}$ hidden neuron of $\phi$ is
\begin{enumerate}[\rm (a)]\itemsep = 0em

\item\label{PropMainLocMinInact} inactive,

\item\label{PropMainLocMinSemiIn1} semi-inactive with $I_j = \{0\}$ and $v_j>0$, or

\item\label{PropMainLocMinSemiIn2} semi-inactive with $I_j = \{1\}$ and $v_j<0$.

\end{enumerate}

\item\label{PropMainSaddle} $\phi$ is a saddle point of $\Loss$ if and only if $\phi$ is centered, $\phi$ does not have any type-1-active neurons, $\phi$ does not have any non-flat semi-active neurons, $\phi$ does not have any non-flat degenerate neurons, and exactly one of the following two items holds:
\begin{enumerate}[\rm (a)]\itemsep = 0em

\item\label{PropMainSaddleTriv} $\phi$ does not have any type-2-active neurons and there exists $j \in \{1,\dots,N\}$ such that the $j^{th}$ hidden neuron of $\phi$ is
\begin{enumerate}[\rm (i)]\itemsep = 0em
\item\label{PropMainSaddleSemiAc} flat semi-active,

\item\label{PropMainSaddleSemiIn1} semi-inactive with $I_j = \{0\}$ and $v_j \leq 0$,

\item\label{PropMainSaddleSemiIn2} semi-inactive with $I_j = \{1\}$ and $v_j \geq 0$, or

\item\label{PropMainSaddleDeg} flat degenerate.

\end{enumerate}

\item\label{PropMainSaddleNonTriv} There exists $n \in \{2,4,6,\dots\}$ such that $(\bigcup_{j \in \{1,\dots,N\},\, w_j \ne 0} \{-\frac{b_j}{w_j}\}) \cap (0,1) = \bigcup_{i=1}^{n} \{\frac{i}{n+1}\}$ and, for all $j \in \{1,\dots,N\}$, $i \in \{1,\dots,n\}$ with $w_j \ne 0 = b_j + \frac{i w_j}{n+1}$, it holds that $\mathrm{sign}(w_j) = (-1)^{i+1}$ and $\ssum{k \in \{1,\dots,N\},\, w_k \ne 0 = b_k +\frac{i w_k}{n+1}}{} v_k w_k = \frac{2}{n+1}$.

\end{enumerate}

\item\label{PropMainFctTriv} If $\phi$ is a non-global local minimum of $\Loss$ or a saddle point of $\Loss$ without type-2-active neurons, then $\NNfct(x) = 1/2$ for all $x \in [0,1]$.

\item\label{PropMainFctNonTriv} If $\phi$ is a saddle point of $\Loss$ with at least one type-2-active neuron, then there exists $n \in \{2,4,6,\dots\}$ such that $n \leq N$ and, for all $i \in \{0,\dots,n\}$, $x \in [\frac{i}{n+1},\frac{i+1}{n+1}]$, one has
\begin{equation}
\label{PropMainFctNonTrivFormula}
	\NNfct(x) = x - \frac{(-1)^i}{n+1} \Big( x - \frac{i+\frac{1}{2}}{n+1} \Big).
\end{equation}%

\end{enumerate}
\end{proposition}

\begin{proof}
	Statement \eqref{PropMainLocMax} follows from \cref{LemNoMaxima} and the `if' part of the `if and only if' statement in \eqref{PropMainLocMin} is the content of \cref{LemCritLocMin}.
Moreover, if $\phi$ is as in \eqref{PropMainSaddleTriv}, then it is a critical point because it satisfies \eqref{LossGradient} and it is a saddle point by \cref{LemCritSaddleT}.
Next, denote $q_i = i/(n+1)$ for all $i \in \{0,\dots,n+1\}$.
If $\phi$ is as in \eqref{PropMainSaddleNonTriv}, then its realization on $[0,1]$ is given by
\begin{equation}
\label{PropIDtarget_ProofFormula}
	\NNfct(x) = \frac{1}{2} + \frac{2}{n+1} \ssum{i=1}{n} (-1)^{i+1} \ReLU{(-1)^{i+1}(x-q_i)}.
\end{equation}%
which coincides with the formula \eqref{PropMainFctNonTrivFormula}.
In particular, we have $\int_{q_i}^{q_{i+1}} (\NNfct(x)-x) dx = 0$ for all $i \in \{0,\dots,n\}$ and $\int_{q_i}^{q_{i+2}} x(\NNfct(x)-x) dx = 0$ for all $i \in \{0,\dots,n-1\}$.
The latter asserts that $\int_{q_i}^{1} x(\NNfct(x)-x) dx = 0$ for odd $i$ and $\int_{0}^{q_i} x(\NNfct(x)-x) dx = 0$ for even $i$.
Thus, $\phi$ satisfies \eqref{LossGradient} and, hence, is a critical point.
Furthermore, it is a saddle point by \cref{LemCritNonAffSaddle}.
This proves the `if' part of the `if and only if' statement in \eqref{PropMainSaddle}.

Now, suppose $\phi$ is a non-global local minimum.
Then $\NNfct$ is affine by \cref{LemCritNonAffSaddle}.
\cref{LemCritAff} asserts that $\phi$ is centered and does not have any active or non-flat semi-active neurons.
Furthermore, for each hidden neuron, \cref{LemCritSaddleT} rules out all possibilities except \eqref{PropMainLocMinInact}-\eqref{PropMainLocMinSemiIn2}.
This proves the `only if' part of \eqref{PropMainLocMin}.

Next, suppose $\phi$ is a saddle point.
If $\NNfct$ is affine, then $\phi$ is centered and does not have any active, non-flat semi-active, or non-flat degenerate neurons by \cref{LemCritAff}.
If there is no hidden neuron as in \eqref{PropMainSaddleSemiAc}-\eqref{PropMainSaddleDeg}, then all hidden neurons satisfy one of the conditions in \eqref{PropMainLocMinInact}-\eqref{PropMainLocMinSemiIn2}.
But this contradicts \cref{LemCritLocMin}.
This proves \eqref{PropMainSaddleTriv}.
If $\NNfct$ is not affine, then it still does not admit any type-1-active or non-flat semi-active neurons by \cref{LemCritGen}.
Moreover, \cref{LemCritNonAff} shows that $\phi$ is centered and its type-2-active neurons satisfy \eqref{PropMainSaddleNonTriv}.
We need to argue that $\phi$ does not have any non-flat degenerate neurons in this case either.
If there were a non-flat degenerate neuron, then $\Grad(\phi) = 0$ implies $0 = \int_0^1 x(\NNfct(x)-x)dx$.
But \cref{LemUniqueAffGeneral}.\eqref{LemUniqueAffGeneralConsequSum} and \cref{LemCritNonAff} ensure that this integral is different from zero.
This finishes the proof of the `only if' part of \eqref{PropMainSaddle}.

Next, we prove \eqref{PropMainDiff}.
If $\phi$ is a saddle point, then it does not have any non-flat degenerate neurons by \eqref{PropMainSaddle}.
If $\phi$ is a non-global local extremum, then \eqref{PropMainLocMax} and \eqref{PropMainLocMin} imply that $\phi$ does not have any non-flat degenerate neurons either.
Thus, $\Loss$ is differentiable at $\phi$ by \cref{LemLossContDiff}.
If $\phi$ is a global minimum, then $\phi$ is point of differentiability by \cref{LemDiffGlobalMinima}.

Statement \eqref{PropMainFctTriv} follows immediately from \eqref{PropMainLocMin} and \eqref{PropMainSaddleTriv}.
The remaining statement \eqref{PropMainFctNonTriv} is implied by \eqref{PropMainSaddleNonTriv} and \eqref{PropIDtarget_ProofFormula}.
\end{proof}


\subsection{Completion of the proof of Theorem \ref{ThrmMain}}
\label{section_main_proof}

In this section, we show that \cref{ThrmMain} can always be reduced to its special case, \cref{PropIDtarget}, by employing a transformation of the parameter space.

\begin{proof}[Proof of \cref{ThrmMain}]
	First, we assume that $T = (0,1)$.
Consider the transformation $P \colon \R^{3N+1} \rightarrow \R^{3N+1}$ of the parameter space given by $P(w,b,v,c) = (w,b,\frac{v}{\alpha},\frac{c-\beta}{\alpha})$.
We then have $\LossFull(\phi) = \alpha^2 \Loss \circ P(\phi)$ for all $\phi \in \R^{3N+1}$.
Since the coordinates $w$ and $b$ remain unchanged and the vector $v$ only gets scaled under the transformation $P$, the transformation $P$ does not change the types of the hidden neurons.
Moreover, a network $\phi \in \R^{3N+1}$ is $(T,\mathcal{A})$-centered if and only if $P(\phi)$ is centered.
The map $P$ clearly is a smooth diffeomorphism and, hence, \cref{ThrmMain} with $T = (0,1)$ is exactly what we obtain from \cref{PropIDtarget} under the transformation $P$.

Now, we deduce \cref{ThrmMain} for general $T$.
This time, set $\mathcal{B} = (\alpha (T_1-T_0), \alpha T_0 + \beta)$ and denote by $Q \colon \R^{3N+1} \rightarrow \R^{3N+1}$ the transformation $Q(w,b,v,c) = ((T_1-T_0)w , T_0w+b ,v,c)$.
Then $\LossFull(\phi) = (T_1-T_0) \Loss_{N,(0,1),\mathcal{B}} \circ Q(\phi)$ for any $\phi \in \R^{3N+1}$.
As above, the transformation $Q$ does not change the types of the hidden neurons.
Note for the breakpoints that
\begin{equation*}
	-\frac{b_j}{w_j} = T_0 + \frac{i(T_1-T_0)}{n+1} \iff -\frac{T_0w_j + b_j}{(T_1-T_0)w_j} = \frac{i}{n+1}.
\end{equation*}%
Also, $\phi \in \R^{3N+1}$ is $(T,\mathcal{A})$-centered if and only if $Q(\phi)$ is $((0,1),\mathcal{B})$-centered.
Since we have shown the theorem to hold for $T = (0,1)$, the smooth diffeomorphism $Q$ yields \cref{ThrmMain} in the general case.
\end{proof}


\section{From ReLU to leaky ReLU}
\label{section_leaky}

In this section, we attempt to derive \cref{ThrmMain} for leaky ReLU activation, given by $x \mapsto \max\{x,\gamma x\}$ for a parameter $\gamma \in (0,1)$.
We denote the realization $\NNfctLeaky \in C(\R,\R)$ of a network $\phi = (w,b,v,c) \in \R^{3N+1}$ with this activation by
\begin{equation*}
	\NNfctLeaky(x) = c + \ssum{j=1}{N} v_j \max\{w_jx+b_j,\gamma (w_jx+b_j)\}.
\end{equation*}%
Analogously to the ReLU case, given $\mathcal{A}=(\alpha,\beta) \in \R^2$ and $T=(T_0,T_1) \in \R^2$, the loss function $\LossLeakyFull \in C(\R^{3N+1},\R)$ is the $L^2$-loss given by
\begin{equation*}
	\LossLeakyFull(\phi) = \int_{T_0}^{T_1} ( \NNfctLeaky(x) - \alpha x - \beta )^2 \, dx.
\end{equation*}%
Again, we call a point a critical point of $\LossLeakyFull$ if it is a zero of the generalized gradient defined by right-hand partial derivatives.
The notions about types of neurons remain the same as in \cref{def_types_of_neurons}.
Strictly speaking, the notions `inactive' and `semi-inactive' are no longer suitable for leaky ReLU activation, but it is convenient to stick to the same terminology.
We will deduce the classification for leaky ReLU by reducing it to the ReLU case in some instances and deal with other instances directly.


\subsection{Partial reduction to the ReLU case}

As before, we first consider the special case where the target function is the identity on $[0,1]$.
Let us abbreviate $\LossLeaky = \LossLeaky_{N,(0,1),(1,0)}$ and $\Loss = \Loss_{2N,(0,1),(1,0)}$.
Let $P \colon \R^{3N+1} \rightarrow \R^{6N+1}$ denote the smooth map $P(w,b,v,c) = (w,-w,b,-b,v,-\gamma v,c)$.
Then, $\NNfctLeaky = \NNfctTransf$ and $\LossLeaky = \Loss \circ P$.
Hence, if $\Loss$ is differentiable at $P(\phi)$, then $\LossLeaky$ is differentiable at $\phi$, so differentiability properties of $\Loss$ convert to $\LossLeaky$.
The partial derivatives of $\LossLeaky$ at any network $\phi$ and any non-degenerate or flat degenerate neuron $j$ are given by
\begin{equation*}
\begin{split}
	\frac{\partial}{\partial w_j} \LossLeaky(\phi) &= \Big( \frac{\partial}{\partial w_j} \Loss \Big) (P(\phi)) - \Big( \frac{\partial}{\partial w_{j+N}} \Loss \Big) (P(\phi)), \\
	\frac{\partial}{\partial b_j} \LossLeaky(\phi) &= \Big( \frac{\partial}{\partial b_j} \Loss \Big) (P(\phi)) - \Big( \frac{\partial}{\partial b_{j+N}} \Loss \Big) (P(\phi)), \\
	\frac{\partial}{\partial v_j} \LossLeaky(\phi) &= \Big( \frac{\partial}{\partial v_j} \Loss \Big) (P(\phi)) - \gamma \Big( \frac{\partial}{\partial v_{j+N}} \Loss \Big) (P(\phi)), \\
	\frac{\partial}{\partial c} \LossLeaky(\phi) &= \Big( \frac{\partial}{\partial c} \Loss \Big) (P(\phi)).
\end{split}
\end{equation*}%
We can also write these in explicit formulas.
To do so, we complement the notation $I_j$ by the intervals $\hat{I}_j = \{x \in [0,1] \colon w_jx + b_j < 0\} = [0,1] \backslash I_j$.
Then,
\begin{equation*}
\begin{split}
	\frac{\partial}{\partial w_j} \LossLeaky(\phi) &= 2 v_j \int_{I_j} x(\NNfctLeaky(x)- x) dx + 2 \gamma v_j \int_{\hat{I}_j} x(\NNfctLeaky(x)- x) dx, \\
	\frac{\partial}{\partial b_j} \LossLeaky(\phi) &= 2 v_j \int_{I_j} (\NNfctLeaky(x)- x) dx + 2 \gamma v_j \int_{\hat{I}_j} (\NNfctLeaky(x)- x) dx, \\
	\frac{\partial}{\partial v_j} \LossLeaky(\phi) &= 2 \int_{I_j} (w_jx+b_j)(\NNfctLeaky(x)- x) dx + 2 \gamma \int_{\hat{I}_j} (w_jx+b_j)(\NNfctLeaky(x)- x) dx, \\
	\frac{\partial}{\partial c} \LossLeaky(\phi) &= 2 \int_0^1 (\NNfctLeaky(x)- x) dx.
\end{split}
\end{equation*}%
This notation allows to treat non-flat degenerate neurons.
For such neurons, the right-hand partial derivatives of $\LossLeaky$ are also given by the above formulas.
We now show how the reduction to the ReLU case works.

\begin{lemma}
\label{LeakyTransfCritPt}
	Suppose $\phi \in \R^{3N+1}$ is a critical point or a local extremum of $\LossLeaky$ but not a global minimum and that $\int_0^1 x(\NNfctLeaky(x)-x) dx = 0$.
Then all neurons of $\phi$ are flat semi-active, flat inactive with $w_j = 0$, or flat degenerate.
\end{lemma}

\begin{proof}
	We first show that $P(\phi)$ is a critical point of $\Loss$ and then apply \cref{ThrmMain} to $P(\phi)$.
Since the partial derivative of $\LossLeaky$ with respect to $c$ exists and must be zero, we have
\begin{equation*}
	\frac{1}{2} \frac{\partial}{\partial c} \LossLeaky(\phi) = \int_0^1 (\NNfctTransf(x)-x)dx = 0 = \int_0^1 x(\NNfctTransf(x)-x)dx.
\end{equation*}%
This shows that the (right-hand) partial derivatives of $\Loss$ are zero at $P(\phi)$ with respect to coordinates corresponding to inactive, semi-inactive, semi-active, type-1-active, and degenerate neurons.
We need to verify that also partial derivatives of $\Loss$ with respect to type-2-active neurons vanish at $P(\phi)$.
To see this, note that, for a type-2-active neuron $j$ of $\phi$, the partial derivative of $\LossLeaky$ with respect to $w_j$ exists at $\phi$ and
\begin{equation*}
	0 = \frac{\partial}{\partial w_j} \LossLeaky(\phi) = 2 (1-\gamma) v_j \int_{I_j} x(\NNfctLeaky(x)- x) dx.
\end{equation*}%
Thus,
\begin{equation*}
\begin{split}
	0 &= 2 v_j \int_{I_j} x(\NNfctLeaky(x)- x) dx = \Big( \frac{\partial}{\partial w_j} \Loss \Big)(P(\phi)), \\
	0 &= - 2 \gamma v_j \int_{\hat{I}_j} x(\NNfctLeaky(x)- x) dx = \Big( \frac{\partial}{\partial w_{j+N}} \Loss \Big)(P(\phi)),
\end{split}
\end{equation*}%
and analogously for the coordinates $b_j,b_{j+N},v_j,v_{j+N}$.
This concludes that $P(\phi)$ is a critical point of $\Loss$.
By \cref{ThrmMain}, $P(\phi)$ does not have any type-1-active, non-flat semi-active, or non-flat degenerate neurons.
By definition of the map $P$, it follows that $\phi$ does not have any type-1-active, non-flat semi-active, or non-flat degenerate neurons, nor does it have any semi-inactive, non-flat inactive, or inactive neurons with $w_j \ne 0$ for otherwise $P(\phi)$ would have one of the former types.
Further, by definition of $P$, any type-2-active neuron of $\phi$ gives rise to two type-2-active neurons of $P(\phi)$ with the same breakpoint but with opposite signs of the $w$-coordinate.
This is not possible by \eqref{ThrmMainSaddleNonTriv} of \cref{ThrmMain}, so $\phi$ cannot have any type-2-active neurons.
In summary, $\phi$ can only have flat semi-active, flat degenerate, or flat inactive neurons with $w_j = 0$.
\end{proof}

The condition $\int_0^1 x(\NNfctLeaky(x)-x) dx = 0$ in the previous lemma is easily converted into a condition about existence of certain types of neurons.
This is done in the first part of the next lemma.
For the second part, we recycle some arguments we learned from the ReLU case.

\begin{lemma}
\label{LemLeakySaddle}
	Suppose $\phi \in \R^{3N+1}$ is a critical point or a local extremum of $\LossLeaky$ but not a global minimum.
Then all neurons of $\phi$ are flat semi-active, flat inactive with $w_j=0$, degenerate, or type-2-active.
Moreover, if $\phi$ does not have any non-flat type-2-active neurons, then $\phi$ is a saddle point and it also does not have any flat type-2-active or non-flat degenerate neurons.
\end{lemma}

\begin{proof}
	Suppose $\phi$ had a neuron of a different type than in the first statement of this lemma, say the $j^{th}$.
Note that one of the intervals $I_j$ and $\hat{I}_j$ is empty and the other one is $[0,1]$ (up to possibly a singleton).
Since the $j^{th}$ neuron is non-degenerate, $\LossLeaky$ is differentiable with respect to the coordinates of the $j^{th}$ neuron, so $\int_0^1 x(\NNfctLeaky(x)-x)dx = 0$.
This contradicts \cref{LeakyTransfCritPt}.

The remainder of the proof is similar to the ones of \cref{LemCritAff,LemCritSaddleT}.
Assume $\phi$ does not have any non-flat type-2-active neurons.
Then $\NNfctLeaky$ is constant on $[0,1]$, and this constant is $1/2$ since $\frac{\partial}{\partial c} \LossLeaky (\phi) = 0$.
We claim that $\phi$ cannot have any flat type-2-active neurons.
Suppose for contradiction the $j^{th}$ neuron were that.
Let $\tau = \mathrm{sign}(w_j)$ and $t_j = -b_j/w_j \in (0,1)$.
Then $\frac{\partial}{\partial v_j} \LossLeaky (\phi) = 0$ implies
\begin{equation*}
	0 = \int_{t_j}^1 (x-t_j)(\tfrac{1}{2}- x) dx + \gamma^{\tau} \int_0^{t_j} (x-t_j)(\tfrac{1}{2}- x) dx = \frac{1}{12} \big( -1 + (1-\gamma^{\tau}) (3-2t_j)t_j^2 \big).
\end{equation*}%
But, for any $\gamma,t \in (0,1)$, $\tau \in \{-1,1\}$, we have $-1 + (1-\gamma^{\tau}) (3-2t)t^2 < 0$, which is a contradiction.
Thus, all neurons of $\phi$ are flat semi-active, flat inactive with $w_j=0$, or degenerate.
With an argument analogous to the proof of \cref{LemCritSaddleT}, we find that $\phi$ is a saddle point of $\LossLeaky$.
Indeed, if there is a flat semi-active or flat inactive neuron $j$ with $w_j = 0$, then, with $\tau = 1-\mathrm{sign}(b_j)$,
\begin{equation*}
	\det
	\begin{pmatrix}
		\frac{\partial}{\partial w_j} \frac{\partial}{\partial w_j} \LossLeaky(\phi) & \frac{\partial}{\partial w_j} \frac{\partial}{\partial v_j} \LossLeaky(\phi) \\
		\frac{\partial}{\partial v_j} \frac{\partial}{\partial w_j} \LossLeaky(\phi) & \frac{\partial}{\partial v_j} \frac{\partial}{\partial v_j} \LossLeaky(\phi)
	\end{pmatrix}
	= - \left( 2 \gamma^{\tau/2} \int_0^1 x(\tfrac{1}{2}-x) dx \right)^2 = - \frac{1}{36} \gamma^{\tau} < 0.
\end{equation*}%
Instead, if there is a degenerate neuron $j$, then, for the perturbation $\phi^s$, $s \in [0,1]$, in the coordinates of the $j^{th}$ neuron given by $w_j^s = \tau s$, $b_j^s = -\tau s^2$, and $v_j^s = v_j + \tau s$ with $\tau = 1$ if $v_j \geq 0$ and $\tau = -1$ if $v_j < 0$, we have
\begin{equation*}
\begin{split}
	\LossLeaky(\phi^s) - \LossLeaky(\phi) &= \frac{1}{6} v_j^s w_j^s \gamma^{(1-\tau)/2} \big( -1 + (1-\gamma^{\tau}) (3-2s)s^2 \big) + \frac{1}{3} (v_j^s w_j^s)^2 \gamma^{1-\tau} \big( (1-s)^3 + \gamma^{2\tau} s^3 \big) \\
	&= -\frac{1}{6} s(|v_j|+s) \gamma^{(1-\tau)/2} + \frac{1}{3}|v_j|^2 s^2 \gamma^{1-\tau} + \mathcal{O}(s^3),
\end{split}
\end{equation*}%
which is strictly negative for small $s > 0$.
This concludes that $\phi$ is a saddle point.
In particular, any degenerate neuron $j$ must be flat because
\begin{equation*}
	0 = \frac{\partial^+}{\partial w_j} \LossLeaky(\phi) = 2 v_j \int_0^1 x(\tfrac{1}{2} - x) dx = - \frac{v_j}{6}.
\end{equation*}%
\end{proof}

We finished dealing with critical points of $\LossLeaky$ that have a constant realization function.
In the next section, we find saddle points of $\LossLeaky$ analogous to the ones in \cref{ThrmMain}.\eqref{ThrmMainSaddleNonTriv}.
For these, we cannot reduce the analysis entirely to the known ReLU case.
However, the arguments are analogous to the ones developed in \cref{LemCritNonAff,LemCritNonAffSaddle}, and we can use a shortcut for small $\gamma$ by arguing that we approximate the ReLU case in a suitable sense.


\subsection{Explicit analysis for leaky ReLU}

The following is the analog of \cref{LemCritNonAff} in the leaky ReLU case.
Informally, one recovers \cref{LemCritNonAff} from \cref{LemLeakyCritNonAff} in the limit $\gamma \rightarrow 0$.
We will discuss this in more detail after having proved the lemma.

\begin{lemma}
\label{LemLeakyCritNonAff}
	Suppose $\phi \in \R^{3N+1}$ is a critical point or a local extremum of $\LossLeaky$ but not a global minimum and that $\phi$ has a type-2-active neuron.
Denote by $0 = q_0 < q_1 < \dots < q_n < q_{n+1} = 1$, for $n \in \N_0$, the roughest partition such that $\NNfctLeaky$ is affine on all subintervals $[q_i,q_{i+1}]$, and denote by $K_i \subseteq \{1,\dots,N\}$ the set of all type-2-active neurons of $\phi$ whose breakpoint is $q_i$.
Then $n \geq 1$ and there exists $\sigma \in \{-1,1\}$ such that, abbreviating
\begin{equation*}
	\delta = \gamma^{(1-\sigma)/4}  + \gamma^{(1-\sigma(-1)^n)/4} + (n-1) \sqrt{1+\gamma},
\end{equation*}%
the following hold:
\begin{enumerate}[\rm (i)]\itemsep = 0em

\item
\begin{enumerate}[\rm (a)]\itemsep = 0em
\item\label{LemLeakyCritNonAffEquidistr} $q_i = q_1 + \frac{(i-1)(q_n-q_1)}{n-1}$ for all $i \in \{2,\dots,n-1\}$,

\item\label{LemLeakyCritNonAffQ1n} $q_1 = \delta^{-1} \gamma^{(1-\sigma)/4}$, and $q_n = 1 - \delta^{-1} \gamma^{(1-\sigma(-1)^n)/4}$, and $q_n-q_1 = \delta^{-1} (n-1) \sqrt{1+\gamma}$,
\end{enumerate}

\item\label{LemLeakyCritNonAffBreakpoints} $-b_j/w_j \in \{q_1,\dots,q_n\}$ for all type-2-active neurons $j \in \{1,\dots,N\}$ of $\phi$,

\item\label{LemLeakyCritNonAffSignW} $\mathrm{sign}(w_j) = \sigma (-1)^{i+1}$ for all $i \in \{1,\dots,n\}$, $j \in K_i$,

\item\label{LemLeakyCritNonAffSlope}
\begin{enumerate}[\rm (a)]\itemsep = 0em
\item $\ssum{j \in K_i}{} v_j w_j =
\begin{cases}
	\gamma^{-1/2} &\text{if } i = 1 = n, \\
	\frac{1}{\delta} \big( \frac{1}{\sqrt{1+\gamma}} + \frac{1}{\gamma^{(1-\sigma)/4}} \big) &\text{if } i=1 \ne n, \\
	\frac{1}{\delta} \frac{2}{\sqrt{1+\gamma}} &\text{if } 2 \leq i \leq n-1, \\
	\frac{1}{\delta} \big( \frac{1}{\sqrt{1+\gamma}} + \frac{1}{\gamma^{(1-\sigma(-1)^n)/4}} \big) &\text{if } i=n \ne 1,
\end{cases}$
\end{enumerate}

\item\label{LemLeakyCritNonAffCentered} $\phi$ is centered,

\item\label{LemLeakyCritNonAffFct} $\displaystyle{\NNfctLeaky(x)-x = \frac{-\sigma (-1)^i (1-\gamma)}{\delta} \cdot
	\begin{cases}
		\frac{x}{\gamma^{(1-\sigma)/4}} - \frac{1}{2\delta} &\text{if } i = 0, \\
		 \frac{x}{\sqrt{1+\gamma}} - \frac{i-1/2}{\delta} - \frac{\gamma^{(1-\sigma)/4}}{\delta \sqrt{1+\gamma}} &\text{if } 1 \leq i \leq n-1, \\
		\frac{x}{\gamma^{(1-\sigma(-1)^n)/4}} + \frac{1}{2\delta} - \frac{1}{\gamma^{(1-\sigma(-1)^n)/4}} &\text{if } i = n
	\end{cases}}$

for all $i \in \{0,\dots,n\}$, $x \in [q_i,q_{i+1}]$.
\end{enumerate}
\end{lemma}

\begin{proof}
	First, note that $\phi$ must have at least one non-flat type-2-active neuron by \cref{LemLeakySaddle}.
For any such neuron $j$,
\begin{equation*}
	0 = \frac{1}{2 v_j} \frac{\partial}{\partial w_j} \LossLeaky(\phi) = (1-\gamma) \int_{I_j} x(\NNfctLeaky(x)- x) dx + \gamma \int_0^1 x(\NNfctLeaky(x)- x) dx,
\end{equation*}%
so the two integrals
\begin{equation}
\label{LemLeakyCritNonAffProofEqu}
\begin{split}
	\int_{I_j} x(\NNfctLeaky(x)- x) dx &= \frac{-\gamma}{1-\gamma} \int_0^1 x(\NNfctLeaky(x)- x) dx, \\
	\int_{\hat{I}_j} x(\NNfctLeaky(x)- x) dx &= \frac{1}{1-\gamma} \int_0^1 x(\NNfctLeaky(x)- x) dx
\end{split}
\end{equation}%
are independent of the non-flat type-2-active neuron $j$.
Doing the same with the coordinate $b_j$ and using that $2\int_0^1 (\NNfctLeaky(x)-x)dx = \frac{\partial}{\partial c}\LossLeaky(\phi) = 0$, we find
\begin{equation}
\label{LemLeakyCritNonAffProofEqu0}
	\int_{I_j} (\NNfctLeaky(x)- x) dx = 0 = \int_{\hat{I}_j} (\NNfctLeaky(x)- x) dx.
\end{equation}%
The function $\NNfctLeaky$ cannot be affine for otherwise we could apply \cref{LemUniqueAffGeneral} with the partition $0 < t_j < 1$ for the breakpoint $t_j$ of any non-flat type-2-active neuron $j$ and obtain a contradiction with $\phi$ not being a global minimum.
In other words, $n \ne 0$.
Moreover, since each $K_i$, $i \in \{1,\dots,n\}$, must contain a non-flat neuron, we deduce from \eqref{LemLeakyCritNonAffProofEqu0} that $\int_{q_i}^{q_{i+1}} (\NNfctLeaky(x)-x)dx = 0$ for all $i \in \{0,\dots,n\}$.
It follows from this and $\frac{\partial}{\partial v} \LossLeaky(\phi) = 0$ that \eqref{LemLeakyCritNonAffProofEqu} holds even for flat neurons $j \in \bigcup_i K_i$.
Also, \cref{LemUniqueAffGeneral} implies that the two integrals in \eqref{LemLeakyCritNonAffProofEqu} are not zero.
In particular,
\begin{equation*}
	\int_{I_j} x(\NNfctLeaky(x)- x) dx \ne \int_{\hat{I}_j} x(\NNfctLeaky(x)- x) dx
\end{equation*}%
for any $j \in \bigcup_i K_i$ and, hence, $\mathrm{sign}(w_{j_0}) = \mathrm{sign}(w_{j_1})$ if $j_0$ and $j_1$ belong to the same set $K_i$.
Furthermore, we find from \eqref{LemLeakyCritNonAffProofEqu} that $\mathrm{sign}(w_{j_0}) \ne \mathrm{sign}(w_{j_1})$ for all $i \in \{1,\dots,n-1\}$ and $j_0 \in K_i$, $j_1 \in K_{i+1}$ by taking differences of the integrals $\int_{I_j} x(\NNfctLeaky(x)- x) dx$ for different $j$.
This establishes item \eqref{LemLeakyCritNonAffSignW}.
Consequently, we obtain from \cref{LemUniqueAffGeneral}.\eqref{LemUniqueAffGeneralConsequSum} (with the partition $q_i,q_{i+1},q_{i+2}$) that
\begin{equation*}
	0 = (q_{i+2}-q_{i+1})^2 - (q_{i+1}-q_i)^2,
\end{equation*}%
for all $i \in \{1,\dots,n-2\}$ (note that we do not obtain this equality for $i=0$ or $i=n-1$).
Thus, the points $q_1,\dots,q_n$ are equidistributed in $[q_1,q_n]$ (but not necessarily in $[0,1]$), which is exactly item \eqref{LemLeakyCritNonAffEquidistr}.
Next, we prove item \eqref{LemLeakyCritNonAffQ1n}.
To do so, we distinguish between even $n$ and odd $n$.
In the former case, $\mathrm{sign}(w_{j_1}) \ne \mathrm{sign}(w_{j_n})$ for all $j_1 \in K_1$, $j_n \in K_n$ by item \eqref{LemLeakyCritNonAffSignW} and, hence, by \eqref{LemLeakyCritNonAffProofEqu},
\begin{equation*}
	\int_0^{q_1} x(\NNfctLeaky(x)-x)dx = \int_{q_n}^1 x(\NNfctLeaky(x)-x)dx.
\end{equation*}%
Write $\NNfctLeaky(x) = A_ix+B_i$ on $[q_i,q_{i+1}]$, for all $i \in \{0,\dots,n\}$, so that the formulas in \eqref{LemUniqueAffGeneralFormulas} hold.
We compute
\begin{equation*}
	\frac{1}{12}(A_0-1)q_1^3 = \int_0^{q_1} x(\NNfctLeaky(x)-x)dx = \int_{q_n}^1 x(\NNfctLeaky(x)-x)dx = \frac{(-1)^n}{12}(A_0-1)q_1(1-q_n)^2.
\end{equation*}%
Thus, $q_1 = 1-q_n$ and, by \eqref{LemLeakyCritNonAffEquidistr},
\begin{equation*}
	\int_0^1 x(\NNfctLeaky(x)-x)dx = \frac{1}{12} (A_0-1) q_1 \sum_{k=0}^{n} (-1)^k (q_{k+1}-q_k)^2 = \frac{1}{12} (A_0-1) q_1 \left( 2q_1^2 - \left(\frac{1-2q_1}{n-1}\right)^2 \right).
\end{equation*}%
Hence, it follows from \eqref{LemLeakyCritNonAffProofEqu} and item \eqref{LemLeakyCritNonAffSignW} that
\begin{equation*}
	q_1^2 = \frac{\sigma\gamma^{(1-\sigma)/2}}{1-\gamma} \left( 2q_1^2 - \left(\frac{1-2q_1}{n-1}\right)^2 \right).
\end{equation*}%
Solving this as a quadratic equation in $q_1$ under the constraint $q_1 \in (0,1/2)$ yields $q_1 = \delta^{-1} \gamma^{(1-\sigma)/4}$.
Now, assume $n$ is odd.
Recall that $\int_{q_i}^{q_{i+2}} x(\NNfctLeaky(x)-x)dx = 0$ for all $i \in \{1,\dots,n-2\}$.
In particular, $\int_{q_1}^{q_n} x(\NNfctLeaky(x)-x)dx = 0$.
Note that $\sigma$ is already determined as the sign of $w_j$ for any $j \in K_1$.
The partial derivative with respect to $w_j$ being zero for a non-flat neuron $j \in K_1$ implies
\begin{equation*}
	0 = \int_{q_n}^1 x(\NNfctLeaky(x)-x) dx + \gamma^{\sigma} \int_0^{q_1} x(\NNfctLeaky(x)-x) dx = -\frac{1}{12} (A_0-1) q_1 ((1-q_n)^2-\gamma^{\sigma}q_1^2).
\end{equation*}%
Thus, $1-q_n = \gamma^{\sigma/2}q_1$.
From this, the formula for $q_1$ follows in the case $n=1$.
If $n \ne 1$, then we use that the partial derivative with respect to $w_j$ for a non-flat neuron $j \in K_2$ is zero to calculate
\begin{equation*}
\begin{split}
	0 &= \int_{q_2}^1 x(\NNfctLeaky(x)-x)dx + \gamma^{-\sigma} \int_0^{q_2} x(\NNfctLeaky(x)-x)dx \\
	&= \int_{q_{n-1}}^1 x(\NNfctLeaky(x)-x)dx + \gamma^{-\sigma} \int_0^{q_2} x(\NNfctLeaky(x)-x)dx \\
	&= \frac{1}{12}(A_0-1)q_1 \left[ \gamma^{-\sigma} q_1^2 - (1-q_n)^2 + (1-\gamma^{-\sigma}) \left( \frac{q_n-q_1}{n-1} \right)^2 \right].
\end{split}
\end{equation*}%
Using $1-q_n = \gamma^{\sigma/2}q_1$, the term in the rectangular brackets becomes a quadratic polynomial in $q_1$, and solving for $q_1$ leads to $q_1 = \delta^{-1} \gamma^{(1-\sigma)/4}$.
This finishes item \eqref{LemLeakyCritNonAffQ1n}.
From here on, we no longer treat even $n$ and odd $n$ separately.
Next, we show item \eqref{LemLeakyCritNonAffBreakpoints}.
Given any type-2-active neuron $j \in \{1,\dots,N\}$, take $i \in \{0,\dots,n\}$ with $q_i \leq t_j \leq q_{i+1}$ and denote $\tau = \mathrm{sign}(w_j)$.
Then, $\frac{\partial}{\partial v_j} \Loss(\phi) = 0$ implies
\begin{equation}
\label{LemLeakyCritNonAffProofEqu2}
\begin{split}
	0 &= \int_{t_j}^{q_{i+1}} (x-t_j) (\NNfctLeaky(x)-x) dx + \gamma^{\tau} \int_{q_i}^{t_j} (x-t_j) (\NNfctLeaky(x)-x) dx \\
	&\quad+ \int_{q_{i+1}}^1 x(\NNfctLeaky(x)-x) dx + \gamma^{\tau} \int_0^{q_i} x(\NNfctLeaky(x)-x) dx.
\end{split}
\end{equation}%
A direct computation with the formulas in \eqref{LemUniqueAffGeneralFormulas} yields
\begin{equation}
\label{LemLeakyCritNonAffProofEqu3}
\begin{split}
	&\int_{t_j}^{q_{i+1}} (x-t_j)(\NNfctLeaky(x)-x) dx + \gamma^{\tau} \int_{q_i}^{t_j} (x-t_j)(\NNfctLeaky(x)-x) dx \\
	&= \frac{(-1)^i}{12} (A_0-1) \frac{q_1}{q_{i+1}-q_i} \Big[ (q_{i+1}-q_i)^3 - (1-\gamma^{\tau}) (t_j-q_i)^2 (3q_{i+1}-2t_j-q_i) \Big].
\end{split}
\end{equation}%
Furthermore, if $i \ne 0$ and $\tau = \sigma (-1)^{i+1}$, then
\begin{equation*}
\begin{split}
	&\int_{q_{i+1}}^1 x(\NNfctLeaky(x)-x) dx + \gamma^{\tau} \int_0^{q_i} x(\NNfctLeaky(x)-x) dx \\
	&= -\int_{q_i}^{q_{i+1}} x(\NNfctLeaky(x)-x) dx + \int_{q_i}^1 x(\NNfctLeaky(x)-x) dx + \gamma^{\tau} \int_0^{q_i} x(\NNfctLeaky(x)-x) dx \\
	&= -\int_{q_i}^{q_{i+1}} x(\NNfctLeaky(x)-x) dx = -\frac{(-1)^i}{12} (A_0-1) q_1 (q_{i+1}-q_i)^2,
\end{split}
\end{equation*}%
where the second-last equality is implied by $\frac{\partial}{\partial w_k} \Loss(\phi) = 0$ for a non-flat type-2-active neuron $k \in K_i$.
Similarly, if $i \ne n$ and $\tau = \sigma (-1)^{i+2}$, then
\begin{equation*}
\begin{split}
	\int_{q_{i+1}}^1 x(\NNfctLeaky(x)-x) dx + \gamma^{\tau} \int_0^{q_i} x(\NNfctLeaky(x)-x) dx &= - \gamma^{\tau} \int_{q_i}^{q_{i+1}} x(\NNfctLeaky(x)-x) dx \\
	&= - \gamma^{\tau} \frac{(-1)^i}{12} (A_0-1) q_1 (q_{i+1}-q_i)^2.
\end{split}
\end{equation*}%
The remaining cases are $i \in \{0,n\}$ with $\tau = -\sigma$, respectively $\tau = \sigma (-1)^n$, for which
\begin{equation*}
	\int_{q_{i+1}}^1 x(\NNfctLeaky(x)-x) dx + \gamma^{\tau} \int_0^{q_i} x(\NNfctLeaky(x)-x) dx =
	\frac{(-1)^{n-i}}{12}(A_0-1)\gamma^{i\tau/n}q_1 \cdot
	\begin{cases}
		\gamma^{(n-i)\sigma/n}q_1^2 &\text{if } n \text{ is odd}, \\
		q_1^2 - (q_2-q_1)^2 &\text{if } n \text{ is even}.
	\end{cases}
\end{equation*}%
In conclusion, we obtain from \eqref{LemLeakyCritNonAffProofEqu2} and \eqref{LemLeakyCritNonAffProofEqu3} that
\begin{equation*}
\begin{split}
	0 =
	\begin{cases}
		-(t_j-q_i)^2 (3q_{i+1}-2t_j-q_i) &\text{if } i \ne 0 \text{ and } \tau = \sigma (-1)^{i+1}, \\
		(q_{i+1}-q_i)^3 - (t_j-q_i)^2 (3q_{i+1}-2t_j-q_i) &\text{if } i \ne n \text{ and } \tau = \sigma (-1)^{i+2}, \\
		(1-\gamma^{\sigma}) q_1^3 - (1-\gamma^{-\sigma})t_j^2(3q_1-2t_j) &\text{if } n \text{ is odd}, i=0, \text{ and } \tau = -\sigma, \\
		(1+\gamma^{\sigma})(1-q_n)q_1^2 - (t_j-q_n)^2(3-2t_j-q_n) &\text{if } n \text{ is odd}, i=n, \text{ and } \tau = -\sigma, \\
		2q_1^3 - q_1(q_2-q_1)^2 - (1-\gamma^{-\sigma})t_j^2(3q_1-2t_j) &\text{if } n \text{ is even}, i=0, \text{ and } \tau = -\sigma, \\
		(1+\gamma^{\sigma}) q_1^3 - \gamma^{\sigma} q_1 (q_2-q_1)^2 - (1-\gamma^{\sigma})(t_j-q_n)^2(3-2t_j-q_n) &\text{if } n \text{ is even}, i=n, \text{ and } \tau = \sigma.
	\end{cases}
\end{split}
\end{equation*}%
In the first case, we must have $t_j = q_i$.
In the second case, the term can be rewritten as $(q_{i+1}-t_j)^2 (q_{i+1}+2t_j-3q_i)$, so we must have $t_j = q_{i+1}$.
In the third case, the two summands always have opposite signs, so their difference is always strictly positive or strictly negative but not zero.
In the fourth case, the right hand side is lower bounded by $(1-q_n)q_1^2$, so it cannot be zero.
In the fifth case, after plugging in $q_1$ and $q_2$, we find that $t_j$ must satisfy
\begin{equation*}
	0 = \sqrt{\gamma} \gamma^{(1+\sigma)/4} + t_j^2 \delta^2 ( 3 \gamma^{(1-\sigma)/4} - 2 t_j \delta ).
\end{equation*}%
However, there is no solution $t_j$ to this equation with $t_j \in [0,q_1]$.
Lastly, in the sixth case, $1-t_j$ must satisfy the same equation, which is incompatible with $t_j \in [q_n,1]$.
This proves item \eqref{LemLeakyCritNonAffBreakpoints}.
Now, we tend to item \eqref{LemLeakyCritNonAffSlope}.
Since $\bigcup_i K_i$ contains all type-2-active neurons of $\phi$ and there are no type-1-active neurons by \cref{LemLeakySaddle}, we can write the slopes of $\NNfctLeaky$ as
\begin{equation}
\label{LemLeakyCritNonAffProofEqu4}
	A_l = \ssum{i=1}{l} \gamma^{\frac{1+\sigma(-1)^i}{2}} \ssum{j \in K_i}{} v_jw_j + \ssum{i=l+1}{n} \gamma^{\frac{1-\sigma(-1)^i}{2}} \ssum{j \in K_i}{} v_jw_j,
\end{equation}%
for all $l \in \{0,\dots,n\}$, by item \eqref{LemLeakyCritNonAffSignW}.
With this, we find, for all $i \in \{1,\dots,n\}$,
\begin{equation*}
\begin{split}
	-\frac{q_1}{q_{i+1}-q_i}(A_0-1) = (-1)^{i-1} (A_i-1) &= (-1)^{i-1} (A_{i-1}-1) + \sigma (1-\gamma) \ssum{j \in K_i}{} v_jw_j \\
	&= \frac{q_1}{q_i-q_{i-1}}(A_0-1) + \sigma (1-\gamma) \ssum{j \in K_i}{} v_jw_j.
\end{split}
\end{equation*}%
Thus, for all $i \in \{1,\dots,n\}$,
\begin{equation*}
	\ssum{j \in K_i}{} v_jw_j = \frac{-\sigma}{1-\gamma} (A_0-1) q_1 \frac{q_{i+1}-q_{i-1}}{(q_{i+1}-q_i)(q_i-q_{i-1})}.
\end{equation*}%
Combining this with the formula \eqref{LemLeakyCritNonAffProofEqu4} for $A_0$ yields
\begin{equation}
\label{LemLeakyCritNonAffProofEqu5}
	\frac{-\sigma(1-\gamma)}{A_0-1} = \sigma(1-\gamma) + q_1 \sum_{i=1}^{n} \gamma^{\frac{1-\sigma(-1)^i}{2}} \frac{q_{i+1}-q_{i-1}}{(q_{i+1}-q_i)(q_i-q_{i-1})} = \gamma^{(1-\sigma)/4} \delta.
\end{equation}%
Plugging this back into the formula for $\ssum{j \in K_i}{} v_jw_j$, we obtain for $n=1$ that $\ssum{j \in K_1}{} v_jw_j = \gamma^{-1/2}$ and for $n \geq 2$, $i \in \{1,\dots,n\}$ that
\begin{equation*}
	\ssum{j \in K_i}{} v_jw_j = \frac{1}{\delta^2} \frac{q_{i+1}-q_{i-1}}{(q_{i+1}-q_i)(q_i-q_{i-1})} =
	\begin{cases}
		\delta^{-1} \big( (1+\gamma)^{-1/2} + \gamma^{-(1-\sigma)/4} \big) &\text{if } i=1, \\
		2 \delta^{-1} (1+\gamma)^{-1/2} &\text{if } 2 \leq i \leq n-1, \\
		\delta^{-1} \big( (1+\gamma)^{-1/2} + \gamma^{-(1-\sigma(-1)^n)/4} \big) &\text{if } i=n.
	\end{cases}
\end{equation*}%
This establishes item \eqref{LemLeakyCritNonAffSlope}.
By the formulas in \eqref{LemUniqueAffGeneralFormulas} and \eqref{LemLeakyCritNonAffProofEqu5},
\begin{equation*}
	A_i-1 = \sigma (-1)^{i+1} (1-\gamma) \delta^{-1} \cdot
	\begin{cases}
		\gamma^{-(1-\sigma)/4} &\text{if } i=0, \\
		(1+\gamma)^{-1/2} &\text{if } 1 \leq i \leq n-1, \\
		\gamma^{-(1-\sigma(-1)^n)/4} &\text{if } i=n
	\end{cases}
\end{equation*}%
and
\begin{equation*}
	B_i = \frac{1}{2} \sigma (-1)^i (1-\gamma) \delta^{-2} \cdot
	\begin{cases}
		1 &\text{if } i=0, \\
		2i-1 + 2 \gamma^{(1-\sigma)/4} (1+\gamma)^{-1/2} &\text{if } 1 \leq i \leq n-1, \\
		2 \gamma^{-(1-\sigma(-1)^n)/4} \delta - 1 &\text{if } i=n.
	\end{cases}
\end{equation*}%
In particular, item \eqref{LemLeakyCritNonAffFct} holds.
Lastly, we know from \cref{LemLeakySaddle} and item \eqref{LemLeakyCritNonAffSignW} that
\begin{equation*}
	0 = \NNfctLeaky(0) - B_0 = c - \sum_{i=1}^n \gamma^{\frac{1-\sigma(-1)^i}{2}} q_i \ssum{j \in K_i}{} v_jw_j - B_0.
\end{equation*}%
After plugging in the formulas for $B_0$, $\delta$, $q_i$, and $\ssum{j \in K_i}{} v_jw_j$, a lengthy but straight-forward computation results in $c = 1/2$.
Thus, $\phi$ is centered, which concludes the proof.
\end{proof}

We make a few remarks about the relationship between the previous lemma and \cref{LemCritNonAff}.
The quantity $\delta$ in \cref{LemLeakyCritNonAff} replaces the factor $n+1$ that appears throughout \cref{LemCritNonAff}.
In the limit $\gamma \rightarrow 0$,
\begin{equation*}
	\delta \rightarrow
	\begin{cases}
		n &\text{if } n \text{ is odd}, \\
		n+1 &\text{if } n \text{ is even and } \sigma = 1, \\
		n-1 &\text{if } n \text{ is even and } \sigma = -1.
	\end{cases}
\end{equation*}%
Thus, in order to match \cref{LemCritNonAff} with the limit case $\gamma \rightarrow 0$ of \cref{LemLeakyCritNonAff}, one would need to apply the former lemma with
\begin{equation*}
	\tilde{n} =
	\begin{cases}
		n-1 &\text{if } n \text{ is odd}, \\
		n &\text{if } n \text{ is even and } \sigma = 1, \\
		n-2 &\text{if } n \text{ is even and } \sigma = -1
	\end{cases}
\end{equation*}%
in place of $n$ so that $\delta \rightarrow \tilde{n} + 1$.
One would hope that the quantities in \cref{LemLeakyCritNonAff} converge to their counterparts from \cref{LemCritNonAff} with $\tilde{n}$ as $\gamma \rightarrow 0$.
Although the number of breakpoints in each lemma is different in most cases (i.e.\ $n \ne \tilde{n}$), this convergence actually happens:
on the one hand, if $n$ is odd and $\sigma = 1$, then $q_n \rightarrow 1$ `degenerates' into the endpoint of the interval $[0,1]$ and only the $(n-1)$-many breakpoints $q_1,\dots,q_{n-1}$ remain, which converge to $\frac{i}{\tilde{n}+1}$, $i \in \{1,\dots,\tilde{n}\}$.
Similarly, if $n$ is odd and $\sigma = -1$, then $q_1 \rightarrow 0$ degenerates and $q_2,\dots,q_n$ remain and converge to the correct breakpoints $\frac{i}{\tilde{n}+1}$, $i \in \{1,\dots,\tilde{n}\}$.
On the other hand, if $n$ is even and $\sigma = 1$, then none of the breakpoints degenerate and $q_1,\dots,q_n$ remain and converge.
Lastly, if $n$ is even and $\sigma = -1$, then both $q_1 \rightarrow 0$ and $q_n \rightarrow n$, and we are left with $q_2,\dots,q_{n-1}$, which converge.

In addition, note that the parity of the $w$-coordinate of the type-2-active neurons match in each lemma even though these are $\sigma (-1)^{i+1}$ and $(-1)^{i+1}$, respectively.
They match because $q_1$ can only degenerate into 0 if $\sigma = -1$.
Lastly, note that the quantities $\ssum{j \in K_i}{} v_j w_j$ also converge to their counterparts as $\gamma \rightarrow 0$.

\begin{lemma}
\label{LemLeakyCritNonAffSaddle}
		Suppose $\phi \in \R^{3N+1}$ is a critical point or a local extremum of $\LossLeaky$ but not a global minimum and that $\phi$ has a type-2-active neuron.
There exists $\gamma_0 \in (0,1]$ depending only on $N$ such that if $\gamma < \gamma_0$, then $\phi$ is a saddle point of $\LossLeaky$.
\end{lemma}

Recall that, in the proof of \cref{LemCritNonAffSaddle}, we studied the Hessian of $\LossLeaky$ restricted to a suitable set of coordinates, taken from type-2-active neurons with breakpoints $\frac{i}{n+1}$, $i \in \{1,2\}$.
To prove \cref{LemLeakyCritNonAffSaddle}, we proceed analogously, which works for sufficiently small $\gamma$ by the above observation about \cref{LemCritNonAff,LemLeakyCritNonAff}.
More precisely, if $n \ne 1$ and $\sigma = 1$, then we will be able to work with the same set of coordinates because $q_1 \rightarrow \frac{1}{\tilde{n}+1}$ and $q_2 \rightarrow \frac{2}{\tilde{n}+1}$.
On the other hand, if $n \geq 3$ and $\sigma = -1$, then $q_1 \rightarrow 0$ but $q_2 \rightarrow \frac{1}{\tilde{n}+1}$ and $q_3 \rightarrow \frac{2}{\tilde{n}+1}$.
In this case, we will use the analogous set of coordinates with $q_2$ and $q_3$ in place of $q_1$ and $q_2$.
However, the argument does not work if $n=1$ or if $n=2$ and $\sigma = -1$ because then $q_1 \rightarrow 0$, $q_2 \rightarrow 1$, and $\NNfctLeaky$ becomes an affine function as $\gamma \rightarrow 0$.
We will treat these two cases separately.

\begin{proof}[Proof of \cref{LemLeakyCritNonAffSaddle}]
	Take $n$, $\delta$, $q_1,\dots,q_n$, and $\sigma$ from \cref{LemLeakyCritNonAff}.
First, assume $n = 2$ with $\sigma = 1$ or $n \geq 3$.
Abbreviate $\tau = (3-\sigma)/2 \in \{1,2\}$.
Similar to the proof of \cref{LemCritNonAffSaddle}, let $K_{\tau} \subseteq \{1,\dots,N\}$ denote the set of those type-2-active neurons with breakpoint $q_{\tau}$, and let $K_{\tau}^- \subseteq K_{\tau}$ be the subset of those neurons $j \in K_{\tau}$ with $v_j < 0$.
Let $j_1 \in K_{\tau}$ with $v_{j_1} > 0$, which exists since $a := \ssum{j \in K_{\tau}}{} v_j w_j > 0$ and $w_j > 0$ for all $j \in K_{\tau}$, and let $j_2,\dots,j_l$, for $l \in \{1,\dots,N\}$, be an enumeration of  $K_{\tau}^-$.
Moreover, let $k \in \{1,\dots,N\}$ be any type-2-active neuron with breakpoint $q_{\tau+1}$.
As in the ReLU case, we consider the Hessian $H$ of $\LossLeaky$ restricted to $(b_{j_1},\dots,b_{j_l},v_k,c)$.

We again introduce some shorthand notation.
For all $i \in \{1,\dots,l\}$, denote $\lambda_i = a^{-1} v_{j_i} w_{j_i}$ so that $\ssum{i=1}{l} \lambda_i \leq 1$.
Define $\mu = \frac{1}{2}(1-(1-\gamma^2)q_{\tau})^{-1} > 0$ and the vectors $u_1 = (v_{j_1},\dots,v_{j_l})$,
\begin{equation*}
	u_2 = \mu
	\begin{pmatrix}
	w_k \big( \gamma (1-2q_{\tau+1}) - (1-\gamma) (q_{\tau+1}-q_{\tau})^2 \big) \\
	2(1-(1-\gamma)q_{\tau})
	\end{pmatrix},
\end{equation*}%
and $u = (u_1,u_2)$.
Furthermore, let $D$ be the diagonal matrix with entries $- \mu (1-\gamma)^2 v_{j_i}^2 / ( a \delta^2 \lambda_i )$, $i \in \{1,\dots,l\}$, let $A$ be the Hessian of $\LossLeaky$ restricted to $(v_k,c)$, let $B = \mu A - u_2u_2^T$, and let $E$ be the diagonal block matrix with blocks $D$ and $B$.
Then $H = \frac{1}{\mu}(E+uu^T)$.
The matrix $A$ is
\begin{equation*}
	A =
	\begin{pmatrix}
	\frac{2}{3} w_k^2 \big( q_{\tau+1}^3 + \gamma^2 (1-q_{\tau+1})^3 \big)
	& -w_k \big( q_{\tau+1}^2 - \gamma (1-q_{\tau+1})^2 \big) \\
	-w_k \big( q_{\tau+1}^2 - \gamma (1-q_{\tau+1})^2 \big) & 2
	\end{pmatrix},
\end{equation*}%
of which both the determinant and the upper left entry are strictly positive.
In particular, $A$ is positive definite and, hence, $\Gamma := \frac{1}{\mu} u_2^T A^{-1} u_2$ is strictly positive.
If $\Gamma < 1$, then the same considerations as in the proof of \cref{LemCritNonAffSaddle} show that $B$ and $E$ are invertible and
\begin{equation*}
	\det(H) = \mu^{-(l+2)} (1 + u_1^TD^{-1}u_1 + u_2^TB^{-1}u_2) \det(E) = \Delta \Big( \frac{a}{\mu} \Big(\frac{\delta}{1-\gamma}\Big)^2 (1-\Gamma) \ssum{i=1}{l} \lambda_i - 1 \Big),
\end{equation*}%
where $\Delta = -\mu^{(l+2)} (1-\Gamma)^{-1} \det(D) \det(B) > 0$.
So far, we did not impose any restrictions on $\gamma$.
To verify that $\Gamma < 1$, we use the limit argument to reduce the calculation to the one we performed in the proof of \cref{LemCritNonAffSaddle}.
To this end, we point out that $\Gamma$ is independent of $w_k$ and that $\delta$, $q_{\tau},q_{\tau+1}$, and $\mu$ only depend on $n$ and $\gamma$.
For fixed $n$, if we let $\gamma$ tend to zero, then $\delta \rightarrow \tilde{n}+1$, $q_{\tau} \rightarrow \frac{1}{\tilde{n}+1}$, $q_{\tau+1} \rightarrow \frac{2}{\tilde{n}+1}$, and $\mu \rightarrow \frac{\tilde{n}+1}{2\tilde{n}}$, where we take $\tilde{n} = n - 1 + \sigma$ if $n$ is even and $\tilde{n} = n-1$ if $n$ is odd.
These limits coincide with the corresponding objects from the proof of \cref{LemCritNonAffSaddle} with $\tilde{n}$ in place of $n$ as discussed prior to stating \cref{LemLeakyCritNonAffSaddle}.
The same goes for the limits of $a$, $u_2$, and $A$.
Thus, we find from \eqref{LemCritNonAffSaddleProofGamma} that, for sufficiently small $\gamma$,
\begin{equation*}
	\Gamma \approx \frac{32\tilde{n}^2-21\tilde{n}+3}{16\tilde{n}(2\tilde{n}-1)} < 1 \quad \text{ and } \quad \frac{a}{\mu} \Big(\frac{\delta}{1-\gamma}\Big)^2 (1-\Gamma) \approx 4\tilde{n}(1-\Gamma) \approx \frac{5\tilde{n}-3}{8\tilde{n}-4} < 1.
\end{equation*}%
This concludes the existence of a  $\gamma_0 \in (0,1]$ such that if $\gamma < \gamma_0$, then $\det(H) < 0$.
This $\gamma_0$ depends only on $n$.
Since $n \leq N$, we can shrink $\gamma_0$ if necessary so that it  depends only on $N$.

It remains to treat the cases $n=1$ and $n=2$ with $\sigma = -1$.
Assume $n = 1$.
This time, let $j_1 \in \{1,\dots,N\}$ be any type-2-active neuron with $\mathrm{sign}(v_{j_1}) = \sigma$, and let $j_2,\dots,j_l$, for $l \in \{1,\dots,N\}$, be an enumeration of all type-2-active neurons with $\mathrm{sign}(v_{j_1}) = -\sigma$.
As before, let $a = \gamma^{-1/2}$, $\lambda_i = a^{-1} v_{j_i} w_{j_i}$, $\mu = \frac{1}{2} \gamma^{-1/2} (1-\sqrt{\gamma}+\gamma)^{-1}$, $D_i = -\mu (1-\gamma)^2 v_{j_i}^2 / ( a \delta^2 \lambda_i )$, and $u_1 = (v_{j_1},\dots,v_{j_l})$ so that $\ssum{i=1}{l} \lambda_i \leq 1$ and $\det(D) < 0$.
On the other hand, let $u_2 = \sigma \mu \sqrt{\gamma} (1-\gamma) \lambda_1 / (\delta^2 v_{j_1})$ and $B = \mu \frac{\partial^2}{\partial v_{j_1}^2} \LossLeaky(\phi) - u_2^2 = \frac{1}{3} \mu^2 \gamma \lambda_1^2 v_{j_1}^{-2} > 0$.
Then the Hessian of $\LossLeaky$ restricted to the coordinates $(b_{j_1},\dots,b_{j_l},v_{j_1})$ is $H = \frac{1}{\mu}(E+uu^T)$, where $E$ is the diagonal block matrix with blocks $D$ and $B$.
Hence,
\begin{equation*}
\begin{split}
	\det(H) &= \mu^{-(l+1)} B \det(D) ( 1 + u_1^T D^{-1} u_1 + u_2^2/B ) \\
	&= - \mu^{-(l+1)} B \det(D) \frac{4(1-\sqrt{\gamma}+\gamma)}{(1+\sqrt{\gamma})^2} \left( \frac{1}{2} \left(\frac{1+\sqrt{\gamma}}{1-\sqrt{\gamma}}\right)^2 \ssum{i=1}{l} \lambda_i - 1 \right).
\end{split}
\end{equation*}%
In particular, $\det(H) < 0$ for sufficiently small $\gamma$.

Lastly, assume $n=2$ and $\sigma = -1$.
Similar as in the beginning, let $K_1 \subseteq \{1,\dots,N\}$ denote the set of those type-2-active neurons with breakpoint $q_1$, and let $K_1^+ \subseteq K_1$ be the subset of those neurons $j \in K_1$ with $v_j > 0$.
Let $j_1 \in K_1$ with $v_{j_1} < 0$, which exists since $a = \ssum{j \in K_1}{} v_j w_j > 0$ and $w_j < 0$ for all $j \in K_1$, and let $j_2,\dots,j_l$, for $l \in \{1,\dots,N\}$, be an enumeration of  $K_1^+$.
Further, denote the same shorthand $\lambda_i = a^{-1} v_{j_i} w_{j_i}$ and $u_1 = (v_{j_1},\dots,v_{j_l})$ but set $\mu = \frac{3}{2}(q_1^3 + \gamma^2 - \gamma^2 q_1^3)^{-1}$ and $D_i = -\mu (1-\gamma)^2 q_1^2 v_{j_i}^2/(a \delta^2 \lambda_i)$.
Then the Hessian of $\LossLeaky$ restricted to $(w_{j_1},\dots,w_{j_l})$ is $H = \frac{1}{\mu}(D+u_1u_1^T)$ with determinant
\begin{equation*}
	\det(H) = \mu^{-l}(1+u_1^T D^{-1} u_1)\det(D) = -\mu^{-l} \det(D) \left( \frac{a \delta^2 }{\mu (1-\gamma)^2 q_1^2} \ssum{i=1}{l} \lambda_i - 1 \right).
\end{equation*}%
By construction, $\ssum{i=1}{l} \lambda_i \leq 1$ and, by plugging in the formulas for $a$, $q_1$, and $\delta$ from \cref{LemLeakyCritNonAff},
\begin{equation*}
	\frac{a \delta^2 }{\mu (1-\gamma)^2 q_1^2} = \frac{2}{3} \frac{\sqrt{1+\gamma} + \sqrt{\gamma}}{(1-\gamma)^2\sqrt{1+\gamma}} (1 + \sqrt{\gamma}\delta^3 - \gamma^2) = \frac{2}{3} \frac{1}{(1-\gamma)^2} + \mathcal{O}(\sqrt{\gamma}).
\end{equation*}%
In particular, $\det(H) < 0$ for small $\gamma$.
\end{proof}


\subsection{Classification for leaky ReLU activation}

In the following, we state the classification of critical points of the $L^2$-loss for leaky ReLU networks.
It is almost analogous to \cref{ThrmMain}, but the main difference is the absence of non-global local minima.
These critical points vanish for leaky ReLU because they were caused solely by dead ReLU neurons.

\begin{theorem}
\label{ThrmMainLeaky}
	Let $N \in \N$, $\phi = (w,b,v,c) \in \R^{3N+1}$, $\mathcal{A}=(\alpha,\beta) \in \R^2$, and $T=(T_0,T_1) \in \R^2$ satisfy $\alpha \ne 0$ and $0 \leq T_0<T_1$.
Then there exists $\gamma_0 \in (0,1]$ such that for all $\gamma \in (0,\gamma_0)$ the following hold:
\begin{enumerate}[\rm (I)]\itemsep = 0em

\item\label{ThrmMainLeakyLocMax} $\phi$ is not a local maximum of $\LossLeakyFull$.

\item\label{ThrmMainLeakyDiff} If $\phi$ is a critical point or a local extremum of $\LossLeakyFull$, then $\LossLeakyFull$ is differentiable at $\phi$ with gradient $\nabla \LossLeakyFull (\phi) = 0$.

\item\label{ThrmMainLeakyLocMin} $\phi$ is not a non-global local minimum of $\LossLeakyFull$.

\item\label{ThrmMainLeakySaddle} $\phi$ is a saddle point of $\LossLeakyFull$ if and only if $\phi$ is $(T,\mathcal{A})$-centered, for all $j \in \{1,\dots,N\}$ the $j^{th}$ hidden neuron of $\phi$ is flat semi-active, flat inactive with $w_j=0$, flat degenerate, or type-2-active, and exactly one of the following two items holds:
\begin{enumerate}[\rm (a)]\itemsep = 0em

\item\label{ThrmMainLeakySaddleTriv} $\phi$ does not have any type-2-active neurons.

\item\label{ThrmMainLeakySaddleNonTriv} There exist $\sigma \in \{-1,1\}$, $n \in \N$ such that if $\delta = \gamma^{(1-\sigma)/4}  + \gamma^{(1-\sigma(-1)^n)/4} + (n-1) \sqrt{1+\gamma}$ and $q_i = T_0+\frac{T_1-T_0}{\delta}\big( \gamma^{(1-\sigma)/4} + (i-1)\sqrt{1+\gamma} \big)$, $i \in \{1,\dots,n\}$, then $\bigcup_{j \in \{1,\dots,N\},\, w_j \ne 0} \{-\frac{b_j}{w_j}\} = \{q_1,\dots,q_n\}$ and, for all $j \in \{1,\dots,N\}$, $i \in \{1,\dots,n\}$ with $w_j \ne 0 = b_j + w_jq_i$, it holds that $\mathrm{sign}(w_j) = \sigma(-1)^{i+1}$ and
\begin{equation*}
	\sum_{\substack{k \in \{1,\dots,N\}, \\ w_k \ne 0 = b_k + w_kq_i}} v_k w_k =
	\begin{cases}
	\frac{\alpha}{\sqrt{\gamma}} &\text{if } i=1=n, \\
	\frac{\alpha}{\delta} \big( \frac{1}{\sqrt{1+\gamma}} + \frac{1}{\gamma^{(1-\sigma)/4}} \big) &\text{if } i=1 \ne n, \\
	\frac{\alpha}{\delta} \frac{2}{\sqrt{1+\gamma}} &\text{if } 2 \leq i \leq n-1, \\
	\frac{\alpha}{\delta} \big( \frac{1}{\sqrt{1+\gamma}} + \frac{1}{\gamma^{(1-\sigma(-1)^n)/4}} \big) &\text{if } i=n \ne 1.
\end{cases}
\end{equation*}%
\end{enumerate}

\item\label{ThrmMainLeakyFctTriv} If $\phi$ is a saddle point of $\LossLeakyFull$ without type-2-active neurons, then $\NNfctLeaky(x) = \frac{\alpha}{2}(T_0+T_1)+\beta$ for all $x \in [T_0,T_1]$.

\item\label{ThrmMainLeakyFctNonTriv} If $\phi$ is a saddle point of $\LossLeakyFull$ with at least one type-2-active neuron, then there exist $\sigma \in \{-1,1\}$, $n \in \N$ such that $n \leq N$ and, for all $i \in \{0,\dots,n\}$, $x \in [q_i,q_{i+1}]$, one has
\begin{equation*}
	\NNfctLeaky(x) - \alpha x - \beta = \left[ \frac{-\sigma (-1)^i (1-\gamma)\alpha}{\delta} \right] \cdot
	\begin{cases}
		\frac{x-T_0}{\gamma^{(1-\sigma)/4}} - \frac{T_1-T_0}{2\delta} &\text{if } i = 0, \\
		\frac{x-T_0}{\sqrt{1+\gamma}} - \frac{(i-1/2)(T_1-T_0)}{\delta} - \frac{\gamma^{(1-\sigma)/4}(T_1-T_0)}{\delta \sqrt{1+\gamma}} &\text{if } 1 \leq i \leq n-1, \\
		\frac{x-T_0}{\gamma^{(1-\sigma(-1)^n)/4}} + \frac{T_1-T_0}{2\delta} - \frac{T_1-T_0}{\gamma^{(1-\sigma(-1)^n)/4}} &\text{if } i = n,
	\end{cases}
\end{equation*}%
where $\delta$ and $q_1,\dots,q_n$ are the same as in item \eqref{ThrmMainLeakySaddleNonTriv}.

\end{enumerate}
\end{theorem}

\begin{proof}
	We prove \cref{ThrmMainLeaky} in the special case $\mathcal{A} = (1,0)$ and $T = (0,1)$.
The general case follows from this the same way as \cref{ThrmMain} followed from \cref{PropIDtarget} in \cref{section_main_proof}.
The first item is shown in \cref{LemNoMaxima}; see \cref{RemNoMaxima}.

Suppose $\phi$ is a critical point or a local extremum of $\LossLeaky$ but not a global minimum.
By \cref{LemLeakySaddle}, all neurons of $\phi$ are flat semi-active, flat inactive with $w_j=0$, degenerate, or type-2-active.
If, in addition, $\phi$ does not have any type-2-active neurons, then it also does not have any non-flat degenerate neurons, it is a saddle point, and $\phi$ must be centered since $\frac{\partial}{\partial c} \LossLeaky(\phi) = 0$.
If, on the other hand, $\phi$ has a type-2-active neuron, then $\phi$ is as in item \eqref{ThrmMainLeakySaddleNonTriv} by \cref{LemLeakyCritNonAff} apart from potentially having non-flat degenerate neurons, and $\phi$ is a saddle point by \cref{LemLeakyCritNonAffSaddle}.
However, a posteriori, $\phi$ cannot have non-flat degenerate neurons because, by \cref{LemLeakyCritNonAff}.\eqref{LemLeakyCritNonAffFct},
\begin{equation*}
	\int_0^1 x(\NNfctLeaky(x)-x)dx = -\frac{(1-\gamma)^2}{12 \delta^4} < 0,
\end{equation*}%
so $\frac{\partial^+}{\partial w_j} \LossLeaky(\phi)$ could not be zero for a non-flat degenerate neuron $j$.
This proves item \eqref{ThrmMainLeakyLocMin} and the `only if' part in item \eqref{ThrmMainLeakySaddle}.
This also implies that any critical point or local extremum of $\LossLeaky$ is a global minimum or does not have any non-flat degenerate neurons.
Hence, the relation $\LossLeaky = \Loss \circ P$ with the smooth map $P$ and the differentiability properties of $\Loss$ assert item \eqref{ThrmMainLeakyDiff}.

If $\phi$ is as in item \eqref{ThrmMainLeakySaddleTriv}, then it clearly is a critical point of $\LossLeaky$, and it is a saddle point by \cref{LemLeakySaddle}.
If $\phi$ is as in item \eqref{ThrmMainLeakySaddleNonTriv}, then $\NNfctLeaky$ is given by the formula in item \eqref{ThrmMainLeakyFctNonTriv}.
We can calculate $\int_{q_i}^{q_{i+1}} (\NNfctLeaky(x)-x)dx = 0$ for all $i \in \{0,\dots,n\}$ and
\begin{equation*}
	 \int_{q_i}^1 x(\NNfctLeaky(x)-x)dx + \gamma^{\sigma(-1)^{i+1}} \int_0^{q_i} x(\NNfctLeaky(x)-x) dx = 0
\end{equation*}%
for all $i \in \{1,\dots,n\}$.
It follows from this that $\phi$ is a critical point of $\LossLeaky$, and it is a saddle point by \cref{LemLeakyCritNonAffSaddle}.
This proves the `if' part in item \eqref{ThrmMainLeakySaddle}.
Item \eqref{ThrmMainLeakyFctTriv} is immediate and the last item was implicit in the previous step.
\end{proof}

\begin{remark}
	The restriction on $\gamma$ to lie in $(0,\gamma_0)$ is only needed in the proof of \cref{LemLeakyCritNonAffSaddle}.
All other proofs were carried out for general $\gamma \in (0,1)$.
We believe that, in fact, one can take $\gamma_0 = 1$ in \cref{LemLeakyCritNonAffSaddle} and, hence, that \cref{ThrmMainLeaky} also holds for general $\gamma \in (0,1)$.
\end{remark}


\section{Classification for quadratic activation}
\label{section_quadratic}

As the last case, we consider the quadratic activation function.
The realization $\NNfctQuad \in C(\R,\R)$ of a network $\phi = (w,b,v,c) \in \R^{3N+1}$ with the quadratic activation is
\begin{equation*}
	\NNfctQuad(x) = c + \ssum{j=1}{N} v_j (w_jx+b_j)^2.
\end{equation*}%
Given $\mathcal{A} = (\alpha,\beta) \in \R^2$ and $T = (T_0,T_1) \in \R^2$, the loss function $\LossQuad$ is the $L^2$-loss given by
\begin{equation*}
	\LossQuad(\phi) = \int_{T_0}^{T_1} (\NNfctQuad(x) - \alpha x - \beta)^2 dx.
\end{equation*}%
This time, there are no issues with differentiability since $\LossQuad$ is infinitely times differentiable, even analytic, everywhere.
The classification turns out to be simpler than in the ReLU and leaky ReLU case as there are no local extrema and only saddle points with a constant realization function.

\begin{theorem}
\label{ThrmMainQuad}
	Let $N \in \N$, $\phi = (w,b,v,c) \in \R^{3N+1}$, $\mathcal{A}=(\alpha,\beta) \in \R^2$, and $T=(T_0,T_1) \in \R^2$ satisfy $\alpha \ne 0$ and $T_0<T_1$.
Then the following hold:
\begin{enumerate}[\rm (I)]\itemsep = 0em

\item\label{ThrmMainQuadLocMax} $\phi$ is not a local maximum of $\LossQuad$.

\item\label{ThrmMainQuadLocMin} $\phi$ is not a non-global local minimum of $\LossQuad$.

\item\label{ThrmMainQuadGlobMin} $\phi$ is a global minimum of $\LossQuad$ if and only if $N \geq 2$ and $\LossQuad(\phi) = 0$.

\item\label{ThrmMainQuadSaddle} $\phi$ is a saddle point of $\LossQuad$ if and only if $\phi$ is $(T,\mathcal{A})$-centered and, for all $j \in \{1,\dots,N\}$, the $j^{th}$ hidden neuron of $\phi$ satisfies $v_jb_j = 0 = w_j$ or $w_j \ne v_j = 0 = b_j + \frac{1}{2}(T_0+T_1)w_j$.

\item\label{ThrmMainQuadFctTriv} If $\phi$ is a saddle point of $\LossQuad$, then $\NNfctQuad(x) = \frac{\alpha}{2}(T_0+T_1)+\beta$ for all $x \in [T_0,T_1]$.

\end{enumerate}
\end{theorem}

\begin{proof}
	As for the other activation functions, the first item is shown in \cref{LemNoMaxima}; see \cref{RemNoMaxima}.
Now, suppose $\phi$ is a critical point of $\LossQuad$ and $\LossQuad(\phi) > 0$.
Since $\LossQuad$ is smooth, we have, for any $j \in \{1,\dots,N\}$,
\begin{equation*}
\begin{split}
	0 &= \frac{\partial}{\partial w_j} \LossQuad(\phi) = 4v_j \int_{T_0}^{T_1} x(w_jx+b_j)(\NNfctQuad(x)-\alpha x-\beta)dx, \\
	0 &= \frac{\partial}{\partial b_j} \LossQuad(\phi) = 4v_j \int_{T_0}^{T_1} (w_jx+b_j)(\NNfctQuad(x)-\alpha x-\beta)dx, \\
	0 &= \frac{\partial}{\partial v_j} \LossQuad(\phi) = 2 \int_{T_0}^{T_1} (w_jx+b_j)^2(\NNfctQuad(x)-\alpha x-\beta)dx, \\
	0 &= \frac{\partial}{\partial c} \LossQuad(\phi) = 2 \int_{T_0}^{T_1}(\NNfctQuad(x)-\alpha x-\beta)dx.
\end{split}
\end{equation*}%
Thus, if there exists $j \in \{1,\dots,N\}$ with $v_j \ne 0 \ne w_j$, then $\int_{T_0}^{T_1} x^m (\NNfctQuad(x)-\alpha x- \beta)dx = 0$ for all $m \in \{0,1,2\}$.
However, note that the zero polynomial is the only polynomial $p$ of degree at most two satisfying $\int_{T_0}^{T_1} x^m p(x)dx = 0$ for all $m \in \{0,1,2\}$.
Hence, since $\LossQuad(\phi) > 0$, we must have $v_j=0$ or $w_j=0$ for all neurons.
In particular, $\NNfctQuad$ is constant and $\int_{T_0}^{T_1} x (\NNfctQuad(x)-\alpha x- \beta)dx \ne 0$.
Thus, for all $j$, if $v_j \ne 0$, then $b_j=0$.
So far, we have shown that all neurons must satisfy $v_j=0$ or $w_j = 0 = b_j$.
It follows that $\phi$ is $(T,\mathcal{A})$-centered.
For a neuron $j$ with $w_j \ne 0$ and $t_j = -b_j/w_j$, we have
\begin{equation*}
	0 = 2 \int_{T_0}^{T_1} (w_jx+b_j)^2(c-\alpha x-\beta)dx = - 2 \alpha w_j^2 \int_{T_0}^{T_1} (x-t_j)^2(x-\tfrac{T_0+T_1}{2})dx,
\end{equation*}%
which is true if and only if $t_j = (T_0+T_1)/2$.
This proves the `only if' direction in \eqref{ThrmMainQuadSaddle}.
Next, we show that $\phi$ must be a saddle point.
We will pick a path $\phi_s = (w^s,b^s,v^s,c^s)$, $s \in (-1,1)$, through $\phi = \phi_0$, which differs only in the coordinates of the first neuron and in
\begin{equation*}
	c^s = c - v_1^s(b_1^s)^2 - \frac{1}{3} A_s (T_0^2 + T_0T_1 + T_1^2) - B_s(T_0+T_1),
\end{equation*}%
where $A_s = v_1^s(w_1^s)^2$ and $B_s = v_1^sw_1^sb_1^s$.
Then,
\begin{equation*}
\begin{split}
	\frac{\LossQuad(\phi_s)-\LossQuad(\phi_0)}{(T_1-T_0)^3} &= \frac{1}{45} A_s^2 (4T_0^2 + 7T_0T_1 + 4T_1^2) + \frac{1}{3} A_sB_s (T_0+T_1) \\
	&\quad + \frac{1}{3} B_s^2 - \frac{\alpha}{6} ( A_s (T_0+T_1) + 2B_s ).
\end{split}
\end{equation*}%
We distinguish three cases.
First, if $v_1 = 0 \ne w_1$, then we use $w_1^s = w_1$, $b_1^s = b_1 - sw_1$, and $v_1^s = -\mathrm{sign}(\alpha) s^2$.
In this case, $B_s = -\frac{1}{2} A_s (T_0+T_1) - sA_s$ and, hence,
\begin{equation*}
	\frac{\LossQuad(\phi_s)-\LossQuad(\phi_0)}{(T_1-T_0)^3} = - \frac{|\alpha|}{3} w_1^2 s^3 + \mathcal{O}(s^4).
\end{equation*}%
This is strictly negative for sufficiently small $s>0$, so $\phi$ is a saddle point.
Secondly, if $v_1 \ne 0 = w_1$, then we use $w_1^s = s$, $b_1^s = -\frac{1}{2}(T_0+T_1)s + \mathrm{sign}(\alpha v_1) s^2$, and $v^s_1 = v_1$.
In this case,
\begin{equation*}
	\frac{\LossQuad(\phi_s)-\LossQuad(\phi_0)}{(T_1-T_0)^3} = - \frac{|\alpha|}{3} |v_1| s^3 + \mathcal{O}(s^4).
\end{equation*}%
In the last case, namely $v_1 = 0 = w_1$, we use $w_1^s = sb_1^s$, $b_1^s = b_1 + s$, and $v_1^s = \mathrm{sign}(\alpha) s^3 (b_1^s)^{-2}$.
Then,
\begin{equation*}
	\frac{\LossQuad(\phi_s)-\LossQuad(\phi_0)}{T_1-T_0} =  - \frac{|\alpha|}{3} s^4 + \mathcal{O}(s^5).
\end{equation*}%
We have shown that if $\phi$ is a critical point with $\LossQuad(\phi) > 0$, then it is a saddle point.
This establishes item \eqref{ThrmMainQuadLocMin} and it also implies that if $\phi$ is a global minimum, then $\LossQuad(\phi) = 0$.
The latter is only possible if $N \geq 2$.
Conversely, if $N \geq 2$, then there are networks with zero loss, so item \eqref{ThrmMainQuadGlobMin} holds.
If $\phi$ is $(T,\mathcal{A})$-centered and all of its neurons are as in item \eqref{ThrmMainQuadSaddle}, then $\nabla\LossQuad(\phi) = 0$ and $\phi$ is a saddle point since clearly $\LossQuad(\phi) > 0$.
This finishes \eqref{ThrmMainQuadSaddle}, and \eqref{ThrmMainQuadFctTriv} follows.
\end{proof}

The conditions in \cref{ThrmMainQuad}.\eqref{ThrmMainQuadSaddle} are equivalent to all neurons being flat semi-active, flat inactive with $w_j=0$, flat type-2-active with breakpoint $-b_j/w_j = (T_0+T_1)/2$, or degenerate.
However, for the quadratic activation, the notions of in-/active neurons seem no longer appropriate.

\begin{remark}
	In \cref{ThrmMainQuad}, the case $N=1$ of a single neuron is special due to the absence of global minima.
The loss can still be arbitrarily small, but there is no network achieving the infimum.
Indeed, for all $(w,b) \in \R^2$ with $w \ne 0$,
\begin{equation*}
	\inf_{(v,c) \in \R^2} \Loss_{1,T,\mathcal{A}}^{\mathrm{quad}}(w,b,v,c) = \frac{1}{12} \alpha^2 (T_1-T_0)^3 \Bigg( 1 - \frac{60\big(\frac{T_0+T_1}{2}+\frac{b}{w}\big)^2}{(T_1-T_0)^2 + 60\big(\frac{T_0+T_1}{2}+\frac{b}{w}\big)^2} \Bigg) \xrightarrow[\mathrm{monotone}]{\frac{T_0+T_1}{2}+\frac{b}{w} \rightarrow \pm \infty} 0.
\end{equation*}%
\end{remark}


\vskip 5mm\noindent{\large\sc Acknowledgments}\vskip 2mm
\noindent The second author acknowledges funding by the Deutsche Forschungsgemeinschaft (DFG, German Research Foundation) under Germany's Excellence Strategy EXC 2044-390685587, Mathematics Muenster: Dynamics-Geometry-Structure and by the startup fund project of Shenzhen Research Institute of Big Data under grant No.\ T00120220001.

\phantomsection%
\addcontentsline{toc}{section}{Bibliography}%
\bibliographystyle{acm}%
\bibliography{bibfile_FR_2022_june_15}

\end{document}